\documentclass[sigconf]{acmart} 

\usepackage{booktabs} 
\usepackage[vlined,ruled,linesnumbered]{algorithm2e}
\usepackage{graphicx}
\usepackage{subfigure}
\usepackage{color}
\usepackage{enumerate}
\usepackage{amsmath}
\usepackage{amsthm}
\usepackage{amsfonts}
\usepackage{multirow}
\usepackage{bbm}
\usepackage{appendix}
\usepackage{cancel}
\AtBeginDocument{%
  \providecommand\BibTeX{{%
    \normalfont B\kern-0.5em{\scshape i\kern-0.25em b}\kern-0.8em\TeX}}}

\copyrightyear{2020} 
\acmYear{2020} 
\setcopyright{acmlicensed}\acmConference[RecSys '20]{Fourteenth ACM Conference on Recommender Systems}{September 21--26, 2020}{Virtual Event, Brazil}
\acmBooktitle{Fourteenth ACM Conference on Recommender Systems (RecSys '20), September 21--26, 2020, Virtual Event, Brazil}
\acmPrice{15.00}
\acmDOI{10.1145/3383313.3412234}
\acmISBN{978-1-4503-7583-2/20/09}

\newtheorem{theorem}{Theorem}
\newtheorem{lemma}{Lemma}
\begin{document}

\title{Contextual User Browsing Bandits for Large-Scale Online Mobile Recommendation}


\author{Xu He, Bo An}
\affiliation{\institution{Nanyang Technological University}}
\email{{hexu0003,boan}@ntu.edu.sg}

\author{Yanghua Li, Haikai Chen}
\affiliation{\institution{Alibaba Group}}
\email{{yichen.lyh,haikai.chk}@taobao.com}

\author{Qingyu Guo}
\affiliation{\institution{Nanyang Technological University}}
\email{qguo005@ntu.edu.sg}

\author{Xin Li, Zhirong Wang}
\affiliation{\institution{Alibaba Group}}
\email{xin.l@alibaba-inc.com, qingfeng@taobao.com}

\renewcommand{\shortauthors}{Xu He et al.}

\begin{abstract}
  Online recommendation services recommend multiple commodities to users. Nowadays, a considerable proportion of users visit e-commerce platforms by mobile devices. Due to the limited screen size of mobile devices, positions of items have a significant influence on clicks: 1) Higher positions lead to more clicks for one commodity. 2) The `pseudo-exposure' issue: Only a few recommended items are shown at first glance and users need to slide the screen to browse other items. Therefore, some recommended items ranked behind are not viewed by users and it is not proper to treat this kind of items as negative samples. 
  While many works model the online recommendation as contextual bandit problems, they rarely take the influence of positions into consideration and thus the estimation of the reward function may be biased. 
  In this paper, we aim at addressing these two issues to improve the performance of online mobile recommendation. Our contributions are four-fold. 
  First, since we concern the reward of a set of recommended items, we model the online recommendation as a contextual combinatorial bandit problem and define the reward of a recommended set. Second, we propose a novel contextual combinatorial bandit method called UBM-LinUCB to address two issues related to positions by adopting the User Browsing Model (UBM), a click model for web search. Third, we provide a formal regret analysis and prove that our algorithm achieves sublinear regret independent of the number of items. Finally, we evaluate our algorithm on two real-world datasets by a novel unbiased estimator. An online experiment is also implemented in Taobao, one of the most popular e-commerce platforms in the world. Results on two CTR metrics show that our algorithm outperforms the other contextual bandit algorithms.
\end{abstract}


\maketitle

\section{Introduction}
With the popularization of e-commerce platforms, a considerable proportion of users visit e-commerce platforms like Amazon and TaoBao by mobile devices.
In typical recommendation scenarios of these platforms, a list of items is recommended online based on the features of items and the profiles of users. Due to the limited screen size of mobile devices, only the first few items of the list are displayed on users' mobile devices at first glance. To view the rest of the list, a user needs to slide the screen to go to the next page. In this process, the user can click an item that attracts him. After viewing the details of the item, he can return to the list to browse other items and make multiple clicks. If none of these items attracts the user in one or two pages, he is likely to quit from the recommendation scenario or the application.

In the aforementioned process, positions of items have a significant influence on clicks, since the probability of examination (a user views an item) is normally larger if the rank of an item is higher, which is called position bias. Moreover, it is possible that the user leaves the scenario without browsing all items and this phenomenon is called pseudo-exposure. For example, if we recommend a list of items to a user, he clicks the first item to view the details of it, and then returns to the recommended list to browse two commodities at position 2 and 3. The rest of commodities are not viewed. However, due to the cost of data transmission and limited computation in mobile devices, it is tricky to accurately know which items are viewed. In many cases, we only know that he clicks the first one. Items behind position 3 are viewed as negative samples leading to biased rewards, since most of the existing bandit researches assume that all the recommended items are browsed by users.
\begin{figure*}[!t]
 \centering
 \subfigure[]{\includegraphics[width=.3\textwidth]{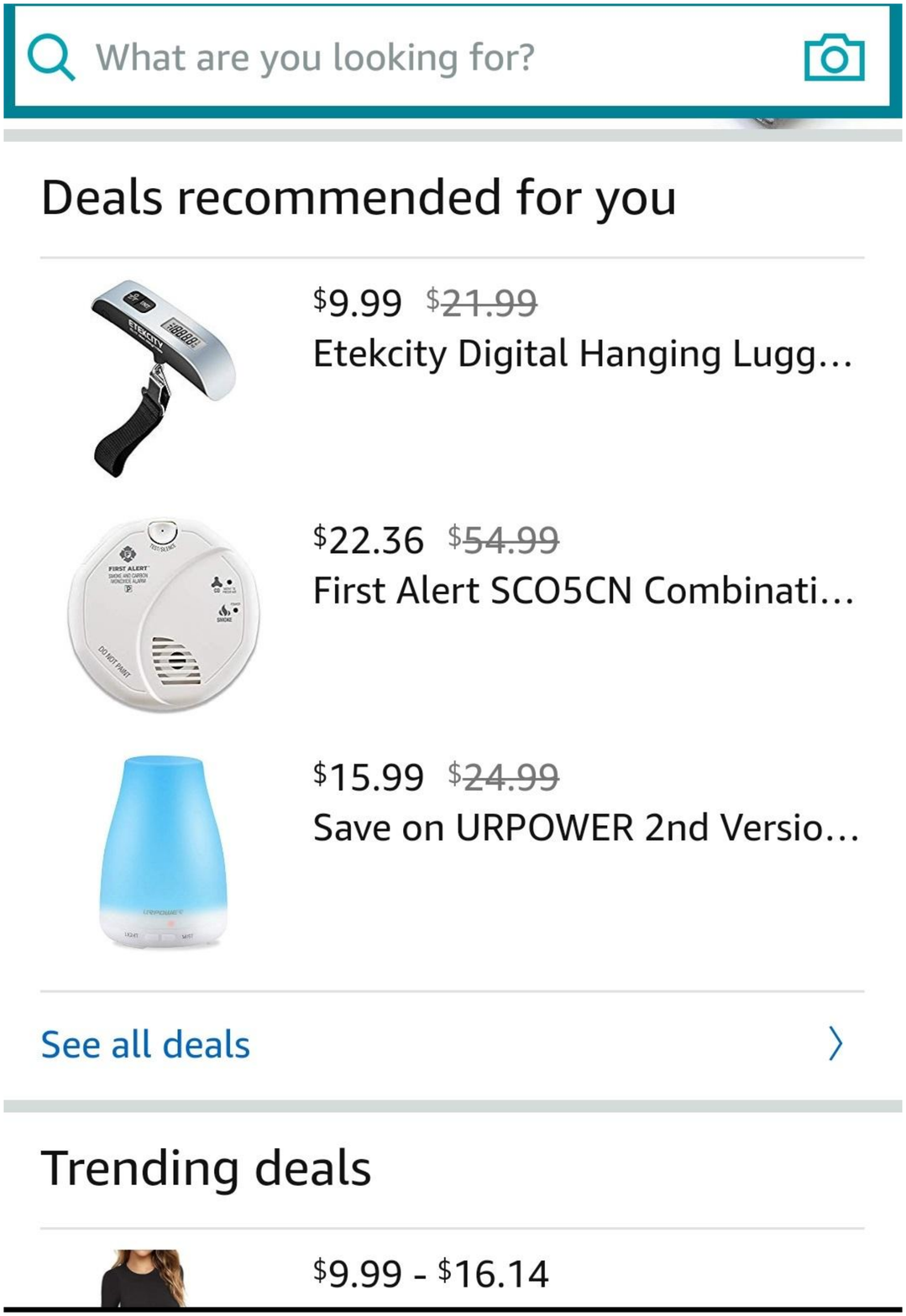}\label{e_figure_1}}\hspace{1cm}
 \subfigure[]{\includegraphics[width=.35\textwidth]{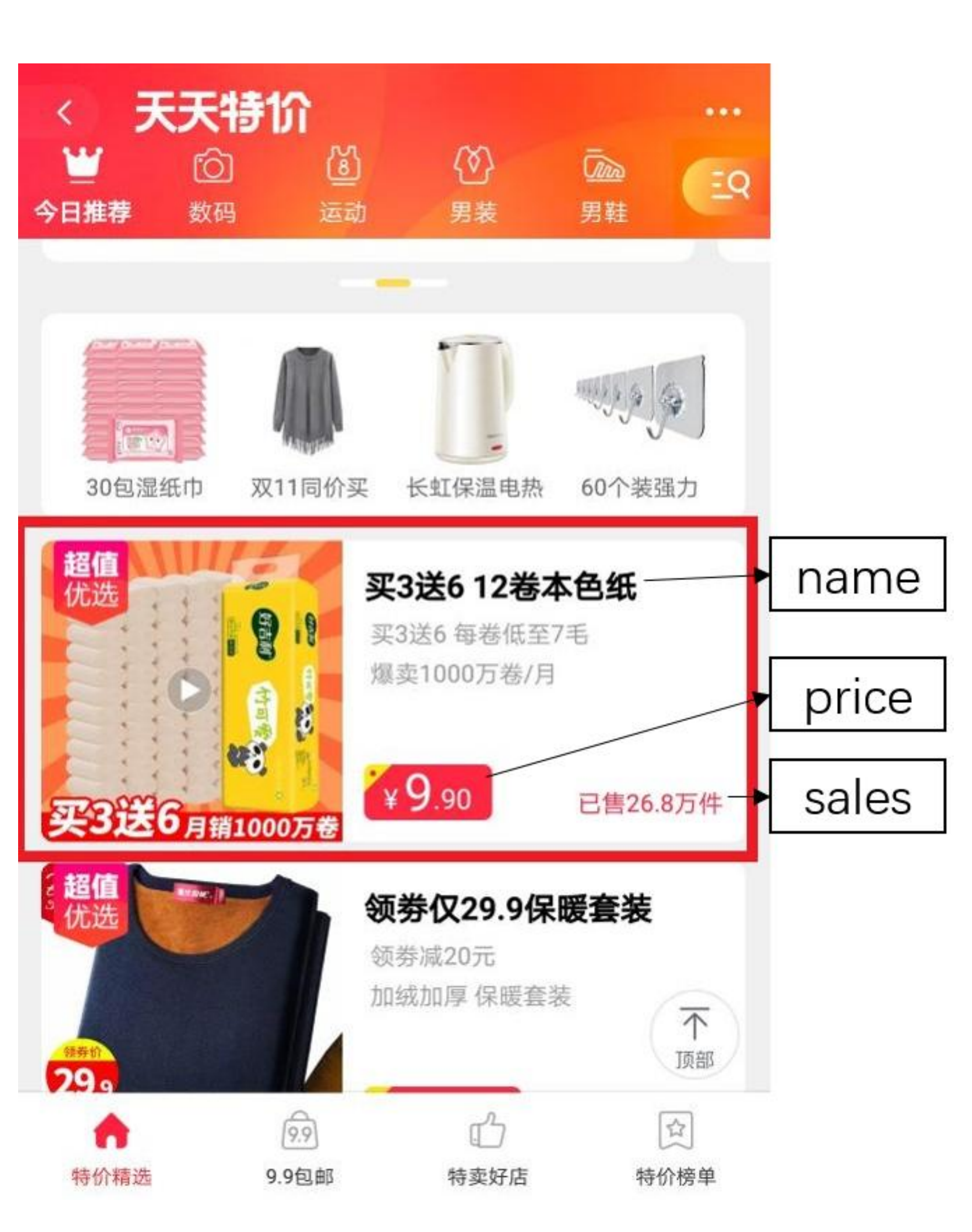}\label{tttj}}
 \caption{Pages in some e-commerce platforms}\label{hybrid_set}
\end{figure*}

This paper aims at addressing the position bias and the pseudo-exposure problem to improve the performance in the online recommendation. We adopt the contextual bandit framework, which is widely used for online recommendation. There are a few challenges that should be addressed to apply the bandit framework in our problem. First, a set of arms should be selected each round since a list of items is recommended in the online recommendation. However, traditional contextual multi-armed bandit algorithms, such as LinUCB \cite{li2010contextual}, only choose one arm each round and thus cannot be applied in our problem directly.
Second, the position bias and pseudo-exposure issue impact our dataset. If we directly learn from such data, clicks are biased due to the influence of positions and some items that are not viewed by users are treated as negative samples. Thus, the learned user preferences could be inaccurate. 
Although position bias is widely studied in recommender systems, only a few works consider it in the bandit domain. 
Though combinatorial bandit algorithms \cite{radlinski2008learning,kakade2008efficient,katariya2016dcm,kveton2015cascading} have been proposed to pull a set of arms in each round, existing works rarely consider the position bias issue in contextual combinatorial bandit methods.
There are some attempts to model the influence of positions on clicks, such as the Cascade Model \cite{craswell2008experimental} and the Position-Based Model \cite{chuklin2015click}. However, the Cascade Model and its extension directly assume that users will not browse the items behind the last click and ignore them, which is not accurate. The existing researches \cite{lagree2016multiple,komiyama2017position} that apply the Position-Based Model in bandit algorithms ignore the features of items, which leads to poor performance in large-scale recommendation problems.
In this paper, we propose a novel contextual combinatorial bandit algorithm to address these challenges. Our contributions are four-fold. First, we model the online recommendation task as a novel contextual combinatorial bandit problem and define the reward of a recommended set.
Second, aiming at addressing the position bias and the pseudo-exposure issue, we assume that the examination probability is determined by both the position of an item and the last click above the position. We estimate the examination probabilities of various positions based on the User Browsing Model and use them as weights of samples in a linear reward model.
Third, we formally analyze the regret of our algorithm. For $T$ rounds, $K$ recommended arms, $d$-dimensional feature vectors, we prove an $\tilde{O}(d\sqrt{TK})$ regret bound. 
Finally, we propose an unbiased estimator to evaluate our algorithm and other baselines using real-world data. An online experiment is implemented in Taobao, one of the largest e-commerce platforms in the world. Results on two CTR metrics show that our algorithm significantly outperforms the extensions of some other contextual bandit algorithms both on the offline and online experiments.


\section{Related Work}
Combinatorial bandit approaches are proposed to extend traditional MAB methods (e.g., UCB, LinUCB) to pull a set of arms in each round. The formulation is first studied in \cite{chen2013combinatorial} to handle the situation where the reward is decided by a set of arms, such as recommendation and advertising. The contextual version of this problem is proposed in which each arm possesses a contextual vector that determines the reward \cite{wen2015efficient,qin2014contextual}.
The idea of utilizing click models in bandit algorithms to model the influence of positions has been explored in prior research works \cite{kveton2015cascading,kveton2015combinatorial,li2016contextual,komiyama2017position,katariya2016dcm,lagree2016multiple,zoghi2017online}.
The Cascade Model \cite{craswell2008experimental} merges with multi-armed bandit, combinatorial bandit, contextual bandit and contextual combinatorial bandit respectively \cite{kveton2015cascading,kveton2015combinatorial,zong2016cascading,li2016contextual}. The Cascade Model assumes that a user scans items from top to bottom until they click one.
Thus, their algorithms cannot be applied to the scenarios where users click more than one item.
To address this disadvantage, \citeauthor{katariya2016dcm} \shortcite{katariya2016dcm} combine a multi-armed bandit algorithm with the Dependent Click Model, an extension of the Cascade Model for multiple clicks scenarios. This model assumes that after users clicked an item, they may still continue to examine other items, if they are not satisfied. However, all these models assume that users will not browse items behind the position of the last click and users' attention on each item before the last click is the same, which is not suitable for our problem. In contrast, the Position-Based Model \cite{lagree2016multiple,komiyama2017position} assume that click rates of items are affected by their displayed positions. \citeauthor{komiyama2017position} \shortcite{komiyama2017position} consider the influence of ads differing from traditional bandit problems.  
However, the pseudo-exposure issue is not considered as an independent issue and the context information of items is ignored in these two researches, which has a negative impact on the application in large-scale problems. 

In summary, the related works mentioned above ignore either the influence of positions or the context of items and users. Thus, in this paper, we consider the influence of positions in contextual combinatorial bandit domain to address the position bias and pseudo-exposure issue.


\section{Problem Statement and Formulation}
In this section, we first provide a problem statement and describe the challenge encountered by mobile e-commerce platforms. We then formally formulate contextual combinatorial bandit problem and give the definition of reward.

With the popularity of smartphones, more and more consumers are prone to visiting e-commerce platforms by mobile applications. Different from web pages, the space of a page on mobile applications is limited and usually only a few commodities can be displayed. For example, Fig. \ref{e_figure_1} and \ref{tttj} display 3 and 2 commodities respectively. Our recommendation scenario is from Taobao, in which 12 commodities are recommended when a user enters the scenario. However, only one item can be seen at the first glance because of the limited screen size of mobile devices. Due to this, the data we obtained is affected by positions. More specifically, the probability that a user examines an item depends heavily on its position. 1) For an item, a higher position leads to a higher probability of click \cite{joachims2005accurately}. 2) Since we only obtain the items that are recommended and clicked after a user leaves our recommendation scenario, whether the items whose positions are behind the last click are viewed or not is unknown, which leads to the pseudo-exposure issue.
An idea for addressing these problems is evaluating the influence of positions based on clicks and positions. The estimated influences are utilized as the weights of samples to learn users' preferences more accurately in this paper.

Recommending multiple commodities to a user based on context information can be naturally modeled as a contextual combinatorial bandit problem. Formally, a contextual combinatorial bandit method $M$ proceeds in discrete rounds $t = 1,2,3,\dots,T$. In round $t$, $M$ observes the current user $u_t$ and a candidate set $\mathcal{A}_t$ including $m$ arms and a set of context $X_t$ that includes $d$-dimensional context vectors $x_{a}$ for each arm $a\in \mathcal{A}_t$. The algorithm $M$ will choose a subset $S_t$ including $K$ arms from $\mathcal{A}_t$ based on the observed information and then receive a reward vector $R_{S_t}$ containing the rewards $r_{a_{k,t}}\in \{0,1\}, \ \forall a_{k,t} \in S_t$, which is the semi-bandit feedback used in existing researches. The algorithm then updates the strategy of choosing arms based on the tuple $(R_{S_t},X_{t})$. 


Now we define the reward functions for the recommended set $S_t$ following \cite{kveton2015cascading,zong2016cascading,liu2018contextual,katariya2016dcm}.
To avoid the abandon of users, we hope that at least one item in the set $S_t$ attracts users and can be clicked. In view of this, the reward of the selected set is defined to be 1, if at least one of the $K$ items is clicked, formally,
\begin{displaymath}
F(S_t) = \left\{ \begin{array}{ll}
1 & \textrm{if $\sum\nolimits_{r \in R_{S_t}} r>0$}\\
0 & \textrm{otherwise}
\end{array} \right.
\end{displaymath}



In our recommendation problem, we view the commodities in the candidate set as arms. When at least one displayed commodity is clicked, the reward $r$ is 1, otherwise 0. With this definition, the expected reward of the recommended set is similar to the definition of click through rate (CTR), which will be introduced in the experiment part. Thus, choosing an item to maximize the CTR of a set is equivalent to maximizing the total expected reward $\mathbb{E}[F(S_t)]$ in our problem.

\begin{algorithm}[t]
\SetKwInOut{Input}{Input}
\SetKwProg{Fn}{Function}{}{}
\caption{LinUCB algorithm with User Browsing Model (UBM-LinUCB)}\label{IBALinUCBSW}
    \Input{Constant $\lambda\geq\phi_w', \ \beta \geq \| \theta^* \|_2^2$ (We set $\lambda = \phi_w', \ \beta = d$ in experiments), weights of different positions $\{w_{k,k'}| k = 1,\dots,K,\ k'=1,\dots,k-1 \}$,the number of items in a page $K$, the set of arms $\mathcal{A}$, the set of context $X$\;}
    \BlankLine
    $A_{0} = \lambda I_d$ ($d$-dimensional identity matrix)\; \label{init_A}
    $b_{0}= 0_{d \times 1}$ ($d$-dimensional zero vector)\;\label{init_b}
    \For{$t=1,\dots,T$}{ 
        $\alpha = \sqrt{d\ln\left(1+ \frac{\phi_w' t}{d\lambda}\right) + 2\ln\left(tK \right)}+\sqrt{\lambda \beta}$\;
        $\theta = A_{t}^{-1}b_{t}$\;
        $p_a = \theta^T x_{a_t} + \alpha\sqrt{x_{a_t}^T A_{t}^{-1} x_{a_t}}$\; \label{selection_LinUCBSW}
        $S_{t} = \varnothing$\;
        \For{$k=1,\dots,K$}{
            $a_{k,t} = \arg\max_{a \in \mathcal{A}_t\backslash S_{t}} p_a$\;
            $S_{t} = S_{t} \cup \{a_{k,t}\}$\; \label{IBA_combine}
        }
        Display $S_t$ to a user\;
        Get reward vector $R_{S_{t}}=[r_{a_{t,1}},\dots,r_{a_{K,t}}]$ of $S_{t}$\; \label{IBA_reward}
        Compute $k'$ for all positions based on $R_{S_{t}}$\;
        
        $A_{t} = A_{t-1} + \sum_{k=1}^K w_{k,k'}^2 x_{a_{k,t}} x_{a_{k,t}}^T$\; \label{update_LinUCBSW1}
        $b_{t} = b_{t-1} + \sum_{k=1}^K w_{k,k'} r_{a_{k,t}} x_{a_{k,t}}$\; \label{update_LinUCBSW2}
    }
\end{algorithm}

\section{LinUCB with User Browsing Model} \label{algorithm}
In this section, we propose our contextual bandit algorithm UBM-LinUCB (LinUCB with User Browsing Model), illustrated in Algorithm \ref{IBALinUCBSW}, which addresses the position bias and  pseudo-exposure issue and is able to recommend a list of commodities with the theoretical guarantee.

After users view some items shown in the screen, they decide to slide the screen or leave, if these items are not attractive. Intuitively, the probability that users leave increases after seeing a long sequence of unattractive items. Thus, the probability of examination decreases when the position of an item becomes lower \cite{joachims2005accurately}. 
However, if an item is clicked, this item must be displayed in the screen. Since the screen usually can display more than one items, the nearby item (behind the clicked one) are more likely to be displayed and the probability of examining increases. Therefore, we introduce position weights relating to both positions and clicks to address the position bias and the pseudo-exposure issue. For all the items, we estimate the probability of examination rather than treating them as negative samples or ignoring them directly.
Inspired by the User Browsing Model (UBM) \cite{chuklin2015click}, we assume that the click through rate $r_{t,a}$ of an arm is determined by the examination probability $w$ and attractiveness $\gamma(a)$ of arm $a$, namely, $r_{t,a}= w_{k,k'} \gamma(a)$, where $\gamma(a)$ is the attractiveness of an item $a$ and $w_{k,k'}$ is the examination probability meaning the probability that a user views an item. We assume that $w_{k,k'}$ depends not only on the rank of an item $k$, but also on the rank of the previously clicked item $k'$.
By assuming that the attractiveness of items follows a linear function, we propose a new linear reward model:
\begin{equation}
   \mathbb{E}[r_{a_k}]=\theta^T (w_{k,k'} x_{a_k})  \label{assumption}
\end{equation}
where $\theta$ is an unknown coefficient vector whose length is the same as $x_{a_k}$. $w_{k,k'}$ is the examination probability for the rank $k$ when the position of the last click is $k'$. We assume that $w_{k,k'}$ is a fixed constant, since it only depends on the structure of a page and can be learned in advance \cite{lagree2016multiple}. We introduce how to obtain $w_{k,k'}$ in the experiment part and focus on our main contribution, bandit algorithm, in this section.

We use ridge regression to solve the linear model \cite{li2010contextual}. The objective of ridge regression is to minimize the penalized residual sum of squares (PRSS):
\begin{equation}
    PRSS(\theta) = \sum\nolimits_{t=1}^T\sum\nolimits_{k=K}^T[r_{a_{k,t}} - \theta^T (w_{k,k'} x_{a_{k,t}})]^2 + \sum\nolimits_{j=1}^d {\theta_j}^2
\end{equation}
where $\theta_{j}$ is the $j$-th element of $\theta$. To simplify notation, we use $w_{k,k'}$ to denote $w_{k,k',a_{k,t}}$, which is the examination probability of the item $a_{k,t}$ at round $t$ where the position of $a_{k,t}$ and the last click before $a_{k,t}$ is $k$ and $k'$ respectively.

Let $X$ be a $TK \times d$-dimensional matrix whose rows correspond to the context $x_{a_{k,t}}$ of the arm $a_{k,t}$ in round $t$ and position $k$. $W$ is also a $TK \times d$-dimensional matrix whose rows are $w_{k,k'}\mathbbm{1}_{1 \times d}$, weights of $a_{k,t}$ in round $t$ and position-pair $(k,k')$. $R$ contains rewards $r_{a_{k,t}}$ for the item $a_{k,t}$ each round $t$ and position $k$. The PRSS function can be transformed to a compact representation:
\begin{equation}
    PRSS(\theta) = [R-\theta^T (W\circ X)]^T [R-\theta^T (W\circ X)]+\lambda\|\theta\|_2^2
\end{equation}
where $\lambda$ is a constant to control the weight of the regularizer and $\circ$ is the Hadamard product and $(W\circ X)_{i,j} = W_{i,j}X_{i,j}$. Thus, each row of $W\circ X$ is $w_{k,k'}x_{a_{k,t}}$.

To minimize the $PRSS(\theta)$, the derivation $\frac{\partial PRSS(\theta)}{\partial \theta}$ should be zero. Then, the solution of $\theta$ is:
\begin{equation}
    \theta = [(W\circ X)^T (W\circ X)+\lambda I_d]^{-1} (W\circ X)^T R
\end{equation}
where $I_d$ is the $d \times d$ identity matrix. Let $A = (W\circ X)^T (W\circ X)+\lambda I_d$ and $b = (W\circ X)^T R$. Applying the online version of ridge regression \cite{li2010contextual}, the update formulations of $A_{t}$ and $b_{t}$ in round $t$ are shown as follows:
\begin{align}
    A_{t+1} &= A_{t} + \sum\nolimits_{k=1}^K w_{k,k'}^2 x_{a_{k,t}} x_{a_{k,t}}^T \label{A_update}\\
    b_{t+1} &= b_{t} + \sum\nolimits_{k=1}^K w_{k,k'} x_{a_{k,t}}^T r_{a_{k,t}} \label{b_update}
\end{align}
The initialization of $A$ and $b$ is shown in Line \ref{init_A} and \ref{init_b}. Then, we use Eq. (\ref{A_update}) and (\ref{b_update}) in Line \ref{update_LinUCBSW1} and \ref{update_LinUCBSW2} respectively to update these two coefficients in our algorithm.

The standard deviation of the ridge regression \cite{walsh2009exploring} for any $a$ and a given pair $\{k,k'\}$ is
$\sqrt{w_{k,k'} x_{a}^T A_{t}^{-1} w_{k,k'} x_{a}}$
and the upper confidence bound used to select the best recommended set is
\begin{equation}
\begin{split}
p_a &= \theta^T w_{k,k'} x_a + \alpha\sqrt{w_{k,k'} x_{a}^T A_{t}^{-1} w_{k,k'} x_{a}} \\
    &= w_{k,k'}(\theta^T x_a + \alpha\sqrt{x_{a}^T A_{t}^{-1} x_{a}})
\end{split}
\end{equation}
where $\alpha$ is a parameter related to $t$, which is defined in the next section. Since the parameter $w_{k,k'}$ for a fixed pair $\{k,k'\}$ is a constant and is not related to $a$, we can ignore it and use the simplified equation in Line \ref{selection_LinUCBSW}:
\begin{equation}
    p_a = \theta^T x_a + \alpha\sqrt{x_{a}^T A_{t}^{-1} x_{a}}
\end{equation}
In the next section, Lemma \ref{lemma_order} indicates how to select the optimal set based on $p_a$ and Theorem \ref{theorem:regret_sum} defines $\alpha$ and constants used in the algorithm.

\section{Theoretical Analysis} \label{Theory}

In this section, we give the theoretical analysis and prove that UBM-LinUCB achieves the regret bound $\tilde{O}(d\sqrt{TK})$ with respect to the aforementioned formulation of reward, where the $\tilde{O}$ notation hides logarithmic factors. \textbf{Our key contributions are 1) the proof for Lemma \ref{lemma_order}, and 2) considering $w_{k,k'}$, which depends on both position $k$ and the position of the last click $k'$ in Theorem \ref{theorem:regret_sum}.}
We define the expected cumulative regret at round $T$ formally: 
\begin{equation}
    R(T) = \mathbb{E}\left[\sum\nolimits_{t=1}^T F(S^*_t)-F(S_t)\right]
\end{equation}
where $S^*_t$ is the optimal set.

We first show that with respect to our novel linear reward model, the optimal set $S^*_t$ at each round $t$ selects the arms with top-$K$ values of $x_{a_t}^T\theta^*$, which holds for $R$ based on the rearrangement inequality. Then, the upper bound is proved. 

Let $\gamma(a)=x_{a}^T \theta^*$. We assume that, without loss of generality, $w_{j+1,j}\geq w_{j+2,j}\geq\dots\geq w_{K,j}$ and $w_{k,k-1}\geq w_{k,k-2} \geq \dots \geq w_{k,0}$. These two assumptions can be explained intuitively: 1) If a user clicks the $j$-th position, the probability that he observes $k$-th ($k>j$) position is inversely related to the distance $k-j$. 2) For a fixed position $k$, the probability of examination is larger when the position of the last click $k'$ is closer to $k$.

The following lemma verifies that the optimal set $S^*_t$ simply selects the arm $a$ with the $k$-th highest value $\gamma(a)$ at the $k$-th slot for $F$ by the rearrangement inequality. Let $S^*_t=\{a^*_{1,t},...,a^*_{K,t}\}$ where $a^*_{k,t}$ is the optimal arm that recommended at the $k$-th position. We have:

\begin{lemma} \label{lemma_order}
$S^*_t$ maximizing $\mathbb{E}[F(S_t)]$ consists of $a^*_{k,t}$ being the arm with the $k$-th highest value $\gamma(a)$.
\end{lemma}

Now we transform the upper bound of $R$ and give the proof.
\begin{lemma}[\cite{katariya2016dcm}]\label{lemma:regret_set}
The upper bound of $R$ is $$\sum\nolimits_{t=1}^T\sum\nolimits_{k=1}^K w_{k,0} \left[\gamma(a^*_{k,t})-\gamma(a_{k,t})\right]$$
\end{lemma}


The proofs of two Lemmas are in the Appendix. Finally, we prove the upper bound of $R(T)$ for UBM-LinUCB algorithm.


\begin{theorem}\label{theorem:regret_sum}
     Let $\phi_w' = \sum_{k=1}^K w_{k,k-1}^2$. When $\lambda \geq \phi_w' \geq 1$, $\| \theta^* \|_2^2 \leq \beta$ and  $$\alpha \geq \sqrt{d\ln\left(1+ \frac{\phi_w' T}{d\lambda}\right) + 2\ln\left(TK \right)}+\sqrt{\lambda \beta},$$ if we run UBM-LinUCB algorithm, then 
    $$ R(T) \leq 2 \alpha \sqrt{2TKd\ln\left(1+\frac{\phi_w' T}{\lambda d}\right)} + 1$$ 
\end{theorem}
\begin{proof} (sketch)
The structure of proof follows \cite{wen2015efficient,abbasi2011improved}. The main contribution is considering $w_{k,k'}$, which depends on both position $k$ and the position of the last click $k'$.
We first define an event to judge the distance between the true attractiveness and the estimated attractiveness of each item $a \in A_t$:
\begin{align*}
E = &\{|\langle x_{a_{k,t-1}}, \theta^*-\theta_{t-1} \rangle | \leq \alpha \sqrt{x_{a_{k,t-1}}^T A_{t-1}^{-1} x_{a_{k,t-1}}}, \\
& \forall {a_{k,t}}\in A_t, \forall t\leq T, \forall k \leq K \}
\end{align*}
If event $E$ happens, the regret can be bounded by the variance 
\begin{align*}
R(T) &\leq 2 \alpha \mathbb{E} \left[ \sum\nolimits_{t=1}^T \sum\nolimits_{k=1}^K w_{k,0} \sqrt{x_{a_{k,t}}^T A_{t-1}^{-1} x_{a_{k,t}}} \right] \\
&\leq 2 \alpha \sqrt{2TKd\ln\left(1+\frac{\phi_w' T}{\lambda d}\right)}
\end{align*}
Then, we prove that the $E$ happens with the probability $1-\delta$ when $\alpha$ satisfies the condition shown in the Theorem. And the bound of regret is $TK\delta$ when $E$ does not happen. Let $\delta = \frac{1}{TK}$ and combining these two parts together, we finish the proof.
\end{proof}

The full version of the proof is in the Appendix. 
We prove that our algorithm can achieve $\tilde{O}(d\sqrt{TK})$ bound and the $\tilde{O}$ notation hides logarithmic factors. Our bound is improved compared to related works \cite{liu2018contextual,zong2016cascading}, which proved $\tilde{O}(dK\sqrt{TK})$ and $\tilde{O}(dK\sqrt{T})$ bounds respectively with the same reward function.


\section{Experimental Evaluation}
In this section, we describe the details of our experiments.  
First, we define the metrics of performance and the evaluation method. An unbiased estimator is introduced.
Second, benchmark algorithms are listed and introduced briefly.
Then, we introduce a public dataset, Yandex Personalized Web Search dataset\footnote{\url{https://www.kaggle.com/c/yandex-personalized-web-search-challenge}}, used in a simulated experiment and a real-world dataset provided by Taobao. We also implement an online experiment in the e-commerce platform. The details of data collection and processing are presented.
Finally, the results of both offline and \textbf{online} experiments are provided and analyzed.




\subsection{Metrics and Evaluation Method}
Corresponding to the formulation of reward, we use the $CTR_{set}$, the expectation of $F(S_t)$, as the metric:
\begin{align*}
    CTR_{set} = \frac{\sum_{t=1}^T F(S_t)}{T}
\end{align*}
\textbf{Additional Metric.} In practice, the CTR of all the recommended items is also widely used. Thus, we define the accumulative reward $CTR_{sum}$ of the set $S_t$, which is the expected total clicks of $S_t$:
\begin{align*}
    CTR_{sum} = \frac{\sum_{t=1}^T \sum_{r \in R_{S_t}}r}{T}
\end{align*}

\textbf{Offline unbiased estimator.} Since the historical logged data is generated by a logging
production policy, we propose an unbiased offline estimator, User Browsing Inverse Propensity Scoring (UBM-IPS) estimator, to evaluate the performance inspired by \citeauthor{li2018offline} \shortcite{li2018offline}. The idea is to estimate users' clicks on various positions based on UBM model and the detail of reduction is at Appendix. The formulation of the UBM-IPS estimator is shown as follows.
\begin{equation}\label{UBMIPS}
\begin{split}
    V_{UBM}(\Phi)= &\mathbb{E}_X\left[ \mathbb{E}_{S\sim \pi(\cdot|X)\atop r\sim D(\cdot|X)}\left[\sum\nolimits_{k=1}^K \sum\nolimits_{k'=0}^k r(a_k,k,k'|X) \right.\right.\\
    &\left.\left. \cdot\frac{\langle \tilde{W}, \Phi(a_k,\cdot,\cdot|X)\rangle}{\langle \tilde{W}, \pi(a_k,\cdot,\cdot|X)\rangle}\right]\right]
\end{split}
\end{equation}
where $D$ is the logged dataset, $a_k\in \mathcal{A}$, $\pi$ is the policy that generate $D$. $\Phi$ is the estimated policy, $\tilde{W} = [w_{1,0},w_{2,0},\dots, w_{K,K-1}]$ is a vector including position weights. Given the features of a candidate set and a user $X$, $\pi(a_k,\cdot,\cdot|X)$ consists of the probabilities that $a_k$ appears at different $k$ and $k'$ under $\pi$ corresponding to $\tilde{W}$ which is $[\pi(a_k,1,0|X), \pi(a_k,2,0|X), \dots, \pi(a_k,K,K-1|X)]$. $\langle \cdot, \cdot \rangle$ means the dot product of two vectors.

The estimator is unbiased when 1) the term $\frac{\langle \tilde{W}, \Phi(a_k,\cdot,\cdot|X) \rangle}{\langle \tilde{W}, \pi(a_k,\cdot,\cdot|X) \rangle}$ is not infinite and 2) users' behaviors follow the UBM model. 
For 1), in our offline experiment, we treat commodities recommended by $\pi$ in one record as candidate set. Thus, all the items chosen by $\Phi$ has been selected by $\pi$ for a specific $X$, that is $\Phi(a_k,\cdot,\cdot|X) \neq 0_{K(K+1)/2}$ and $\pi(a_k,\cdot,\cdot|X) \neq 0_{K(K+1)/2}$. Thus, the UBM-IPS term is not infinite.
For 2), we use the evaluation method proposed in \cite{chapelle2009dynamic} to empirically estimate the accuracy of the UBM model in our recommendation dataset. This evaluation method removes the position bias by using the data in the first position as the test set and the rest as the training set. The experiment result in our e-commerce dataset shows that the UBM-IPS estimator is better than traditional IPS estimator that does not involve the UBM model (MSE: UBM-IPS: 0.1754, IPS: 0.1943), although the MSE is not 0 (totally unbiased). Thus, we use UBM-IPS to evaluate different algorithms empirically in the offline experiments.

Since $k'$ depends on previous clicks, $\Phi(a_k,\cdot,\cdot|X)$ cannot be obtained directly. In practice, we first generate a recommended set by $\Phi$. Since $\Phi$ is deterministic in our case, the recommended set determines $\Phi(a_k,\cdot,\cdot|X)$. Then, we obtain $k'$ sequentially. More specifically, when the policy $\Phi$ chooses the set $S_{\Phi}$ after sampling a context set $X$ from dataset, we first collect $|D(\cdot|X)|$ records $D(\cdot|X)$ with the same context set $X$. Then, for the first commodity $a_1\in S_{\Phi}$, delete the terms that do not include $a_1$ in the Eq. (\ref{UBMIPS}):
\begin{align} \label{ubm_final}
\frac{1}{|D(\cdot|X)|}\sum\nolimits_{a_1\in D(\cdot|X)} r(a_1,k,k'|X)\frac{\langle \tilde{W}, \Phi(a_1,\cdot,\cdot|X)\rangle}{\langle \tilde{W}, \pi(a_1,\cdot,\cdot|X)\rangle} 
\end{align}
This equation is used to simulate the reward $r(a_1,1,0)$, where $k$ and $k'$ are the position of $a_1$ and the last click before $k$ in records respectively. Since $\Phi$ is deterministic, $\Phi(a_1,\cdot,\cdot|X) = \left\{\begin{array}{cl}
   1  & for \ k=1,k'=0  \\
   0  & others
\end{array} \right.$. Notice that the simulated reward would be larger than one since ${\langle \tilde{W}, \Phi(a_1,\cdot,\cdot|X)\rangle}$ would be larger than ${\langle \tilde{W}, \pi(a_1,\cdot,\cdot|X)\rangle}$. Then, we obtain the $k'$ for the second commodity $a_2$ based on reward $r(a_1,1,0)$:
$$k' = \left\{\begin{array}{cl}
   1  & r(a_1,1,0) \geq 1 \ or \ \mathcal{B}(r(a_1,1,0)) = 1  \\
   0  & \mathcal{B}(r(a_1,1,0)) = 0
\end{array} \right.$$
where $\mathcal{B}(\cdot)$ is the reward sampled from the Bernoulli distribution with mean $r(a_1,1,0)$. Given a fixed $k'$, we can use the above method to compute the reward of $a_2$. Repeating this process, the rewards of all commodities in set $S_{\Phi}$ can be obtained. Then, the $CTR_{sum}$ and $CTR_{set}$ can be calculated. For the $CTR_{sum}$, we directly use the average of the unbiased rewards over $T$. For the $CTR_{set}$, we use the the Bernoulli distribution to sample a reward (0 or 1) if $r\in[0,1]$. If the reward of a recommended item is not less than 1, $F(S_t) = 1$.

\textbf{Influence of Positions.} The influence of position is mainly determined by the structure of a page. Since a page's structure does not change in the long term, the position influence $w_{k,k'}$ can be considered as constants following \cite{lagree2016multiple}. We use the EM algorithm suggested by \cite{chuklin2015click,dempster1977maximum} to estimate $w_{k,k'}$ in advance. The max and min values are 0.55 and 0.17 respectively for our recommendation scenario. For the Yandex dataset, they are 0.96 and 0.015. The difference is caused by the distribution of clicks in these two datasets. Since the Yandex dataset is collected from a web search scenario, $79.3\%$ of users click the first web page shown in the result, which is usually the most relevant to the keyword they queried. Clicks are very rare for low-ranking web pages. Therefore, the gap between the max and min position influences is more significant. However, for our recommendation scenario, users usually are not purposeful and prone to browse more items (see Fig. \ref{pie}). Thus, the gap of position influences is relatively small.

\subsection{Benchmark Algorithms}
We evaluate 4 baseline methods to illustrate the performance of our proposed algorithm. We only compare with combinatorial bandit algorithms that are based on LinUCB considering fairness and the problem setting. The first one is a context-free combinatorial bandit algorithm merged with the PBM model. The other three algorithms are contextual combinatorial bandit algorithms. Theoretical optimal parameters are used in experiments.
\begin{itemize}
\item \textbf{PBM-UCB} \cite{lagree2016multiple}: The PBM-UCB algorithm combines the PBM model with the UCB algorithm. Positions bias is considered to compute the mean and variance of an arm's reward by applying the position-based model (PBM), which assumes that the position bias is only related to the rank of an item. The position bias is obtained by the EM algorithm mentioned in \cite{lagree2016multiple}.

\item \textbf{C\textsuperscript{2}UCB} \cite{qin2014contextual}: C\textsuperscript{2}UCB algorithm is a combinatorial extension of the LinUCB algorithm. In each round, the algorithm selects $K$ items with top-$K$ upper confidence bounds.
After displaying them to a user, the algorithm can observe the rewards for these $K$ items and execute the update equations for each of these items.

\item \textbf{CM-LinUCB}: CM-LinUCB is a simple extension of \cite{li2016contextual,zong2016cascading} based on the Cascading Model. This model assumes that a user scans items from top to bottom until they find a relevant item. After they click the relevant item, they will leave without view other items behind. Thus, we ignore items behind the first click and only use the rest of samples to update parameters in this extension, which is the same as the First-Click algorithm used in \cite{katariya2016dcm}.

\item \textbf{DCM-LinUCB} \cite{liu2018contextual}: DCM-LinUCB is an contextual bandit algorithm based on Dependent Click Model. DCM is an extension of CM by assuming that after a user clicked an item, they may still continue to examine other items. The model introduces a satisfaction variable to determine whether the user will continue to view other items or not. If the satisfaction variable equals to 1, the user will leave. Otherwise, he will continue to scan other items.
\end{itemize}

\begin{figure*}[!t]
 \centering
 \subfigure[$CTR_{sum}$]{\includegraphics[width=.4\textwidth]{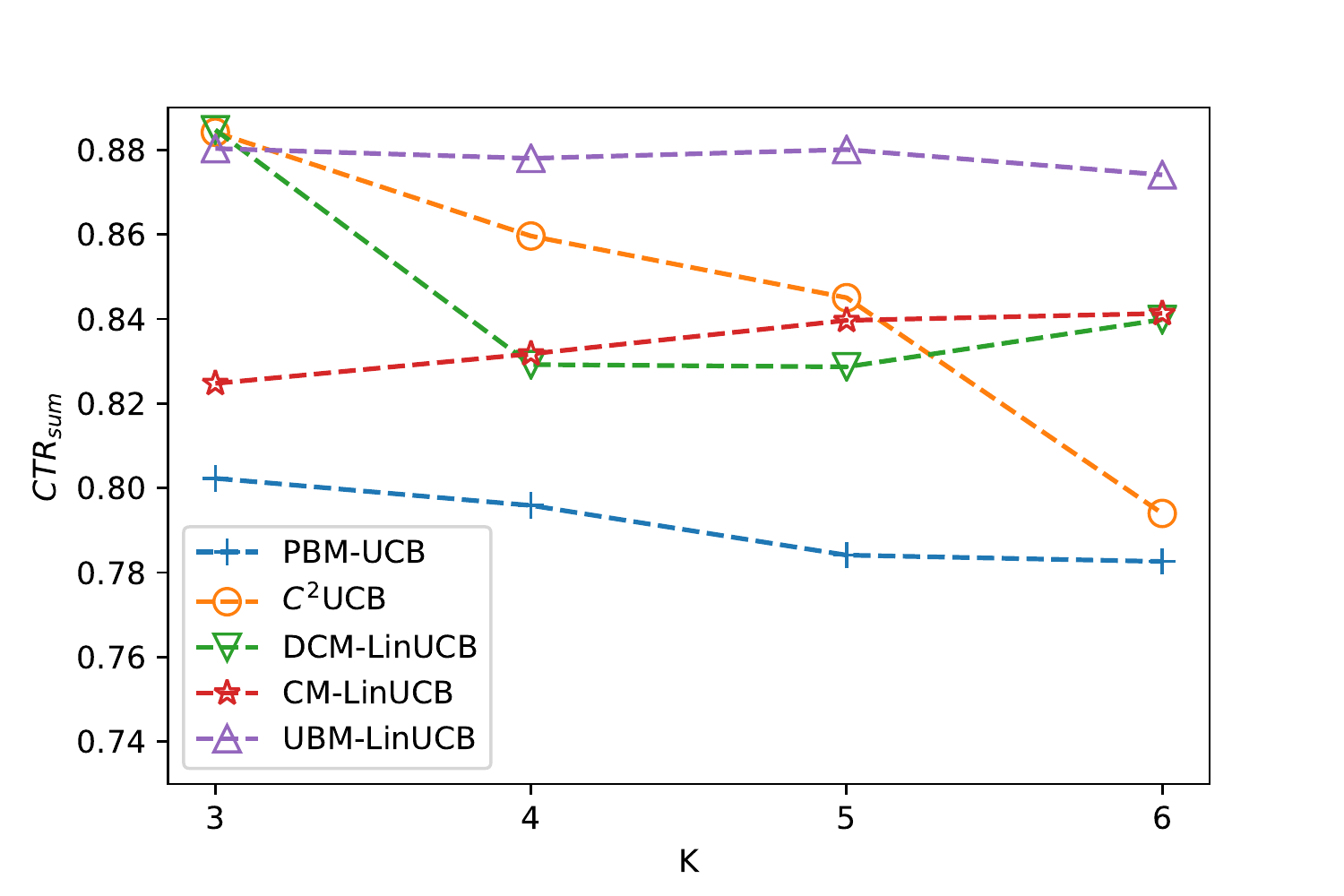}\label{yandex_sum}}
 \subfigure[$CTR_{set}$]{\includegraphics[width=.4\textwidth]{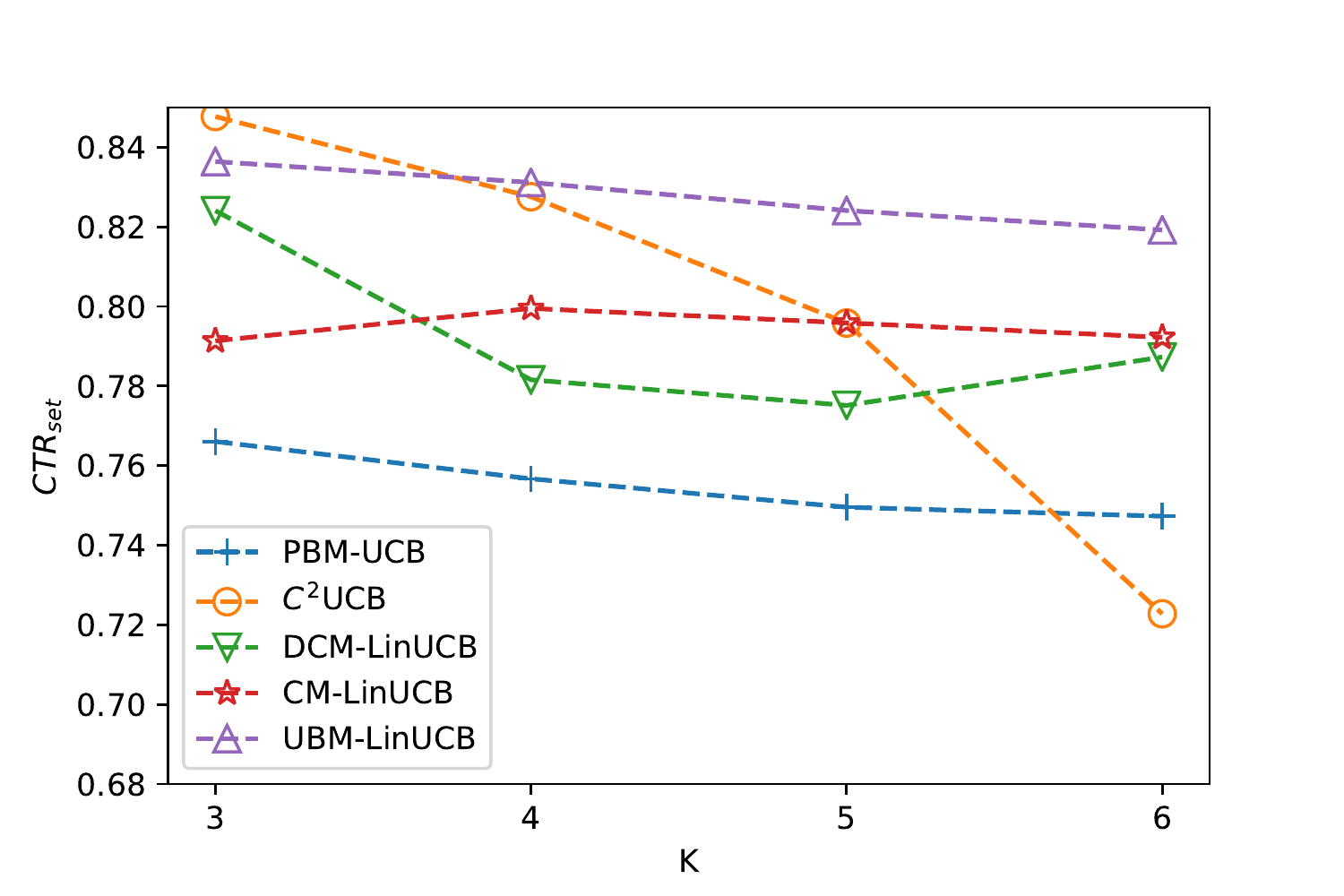}\label{yandex_set}}
 \caption{The performances of algorithms when $K$ changes.} \label{yandex_fig}
\end{figure*}

\subsection{Web Search Recommendation}
The Yandex Personalized Web Search dataset contains 35 million search sessions. Each session contains a user ID, a query ID, IDs of web pages shown as result, and clicks. The top 3 most frequent queries are used for evaluation. The position influences are estimated by the PyClick library \cite{chuklin2015click}. Since the web pages in the dataset do not have corresponding features, we build features for all the web pages. We first construct a user-website matrix $M_{u\times m}$ based on the records in the dataset, where $u=194749$ and $m=2701$ is the number of users and web pages respectively. The value of the element $M(i,j)$ is determined by the attractiveness of a web page $j$ for a user $i$ estimated by a part of Eq. \ref{ubm_final}:
\begin{align} 
\frac{1}{|D(\cdot|X)|}\sum\nolimits_{a_{j}\in D(\cdot|X)\rangle} r(a_{j},k,k'|X)\frac{1}{\langle \tilde{W}, \pi(a_{j},\cdot,\cdot|X)\rangle} 
\end{align}
where $X=i$ represents the user $i$. Thus, the simulated reward can be obtained by multiplying the position weight ${\langle \tilde{W}, \Phi(a_1,\cdot,\cdot|X)\rangle}$ given any policy $\Phi$. The features are generated by the truncated randomized SVD provided by Scikit-Learn library\footnote{\url{https://scikit-learn.org/stable/}}, i.e., $M \approx USV^T$. We set the dimension of $S$ to 10 and use the vector $[U(i),V(j)]$ as the feature of web page $j$ for the $i$-th user. Notice that $|S|=10$ is a balance between accuracy and complexity. Since the truncated randomized SVD is an approximated decomposition method, a smaller $|S|$ leads to worse accuracy. For example, all the contextual bandit algorithms cannot perform normally when $|S|=5$. And a larger $|S|$ increases the computational complexity since the inverse of $A_k$ need $O(d^3)$ time in each round, where $d$ is the dimension of the feature vector. 

For each round, a user is randomly chosen to decide the attractiveness of different web pages. We only use the web pages that are shown to the user as the candidate set, since we cannot know the attractiveness of a web page if it is never displayed to the user. The simulated reward is estimated by the estimator we mentioned above using Eq. \ref{ubm_final}.

\subsubsection{Result}
The result after 5000 rounds is shown in Fig. \ref{yandex_fig} (average of 10 runs). When $K=3$, contextual bandit algorithms perform similarly except CM-LinUCB. The reasons are: 1) The influence of positions is relatively small when $K=3$. Thus, the effectiveness of different models cannot be revealed well. 2) Due to the hypothesis of the Cascading Model, CM-LinUCB only uses samples not behind the first click to update, which leads to insufficient training. 
When $K$ becomes larger, the performance of C\textsuperscript{2}UCB declines dramatically and other contextual bandit algorithms outperform it, since C\textsuperscript{2}UCB does not consider the impact of positions. When $K$ is larger than 4, the CTRs change slightly and do not increase anymore, since $91.6\%$ of users click one web and $98.4\%$ of users clicks 2 web pages or less. Moreover, the ratio of positive samples descends rapidly and has a negative impact on the learning results, especially for the method not involving click models.

\begin{figure}[!t]
 \centering
 \includegraphics[width=.35\textwidth]{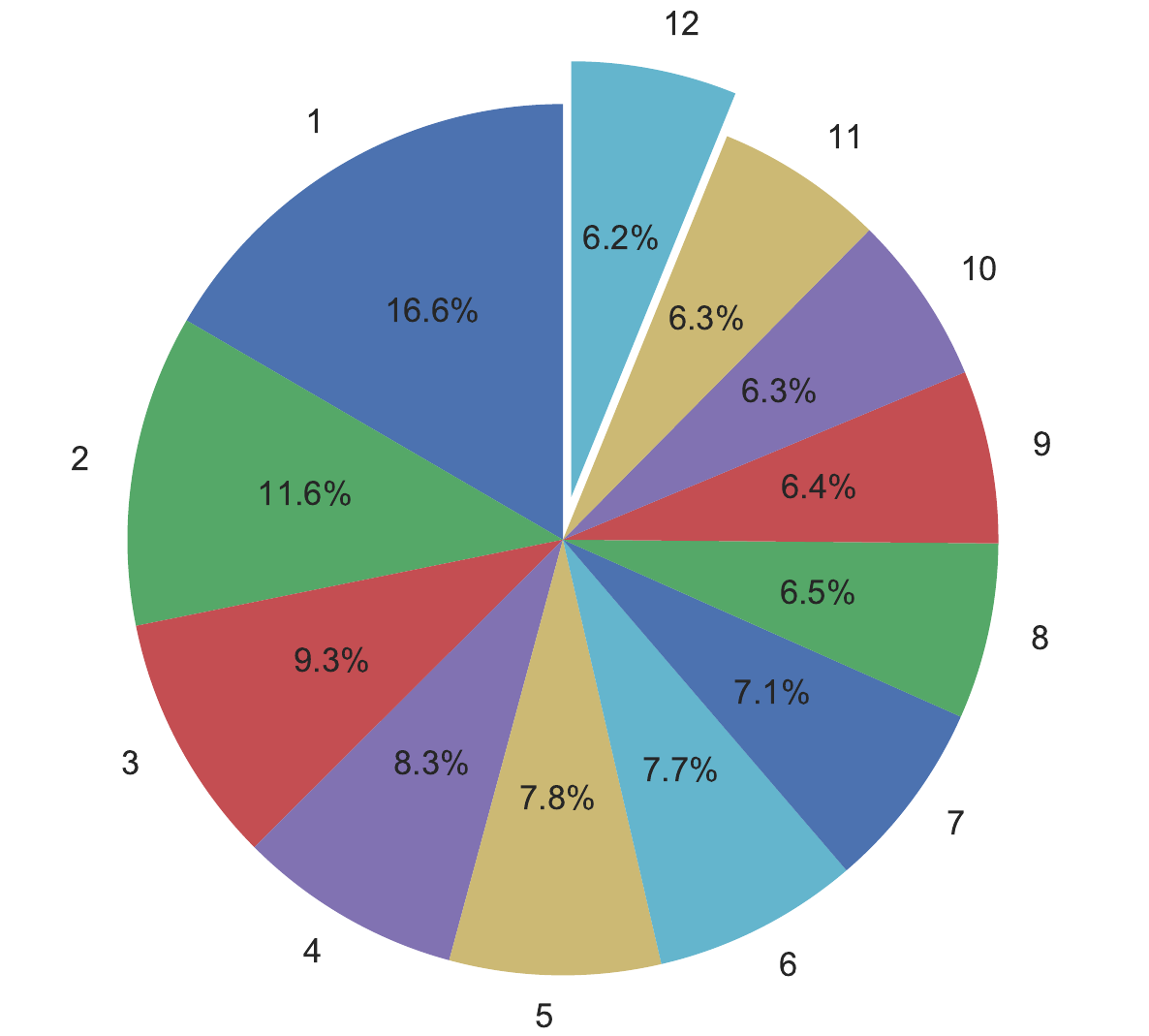}
 \caption{The distribution of the last clicks' positions. The numbers outside the pie are positions of the last clicks. For example, 1 means that the last clicks of users are at the first position.} \label{pie}
\end{figure}

\subsection{E-commerce Recommendation}
Our real-world dataset is provided by Taobao. We utilize the data from a recommendation scenario in the platform, where 12 items are selected and then displayed to users. There are two special characteristics of this scenario. First, the number of commodities is fixed to 12. The mobile app of this platform will not recommend new items after all of these 12 items are browsed. Another characteristic is that users can only view the first item and part of the second item when they enter this scenario. To observe the rest of the items, users need to slide the screen.

Since we cannot know a user's real interest if he does not click any item, we remove the data without clicks from the original log data. After processing, our dataset contains 180k records, while each record contains a user ID, IDs of 12 commodities, context vectors of 12 commodities and a click vector corresponding to the 12 commodities. The context vector is a $56 \times 1$ array including some features of both a commodity and a user, such as the commodity's price and the user's purchasing power. In our offline experiment, we treat 12 commodities in each record as the candidate set since whether users will click other commodities is unknown.


To illustrate the position bias and the pseudo-exposure issue better, Fig. \ref{pie} demonstrates the distribution of the last clicks' positions. The number around the circle is the position of the last click. Only 6.2\% of the last items are clicked by users in our data.
If we assume that the items at the positions behind the last click are affected by the pseudo-exposure, $\frac{11}{12}$ of items are biased for the records whose last click's position is 1, which accounts for 16.6\% of the data. By calculating the weighted sum $\frac{11}{12} \times 16.6\% + \frac{10}{12} \times 11.6\% + \dots $, about 54\% of rewards are affected by the pseudo-exposure issue. Notice that this percentage is a rough estimate, since it is possible that some items at positions behind the last click are viewed by users.

\begin{figure*}[!t]
 \centering
 \subfigure[$K=3$, $CTR_{sum}$]{\includegraphics[width=.24\textwidth]{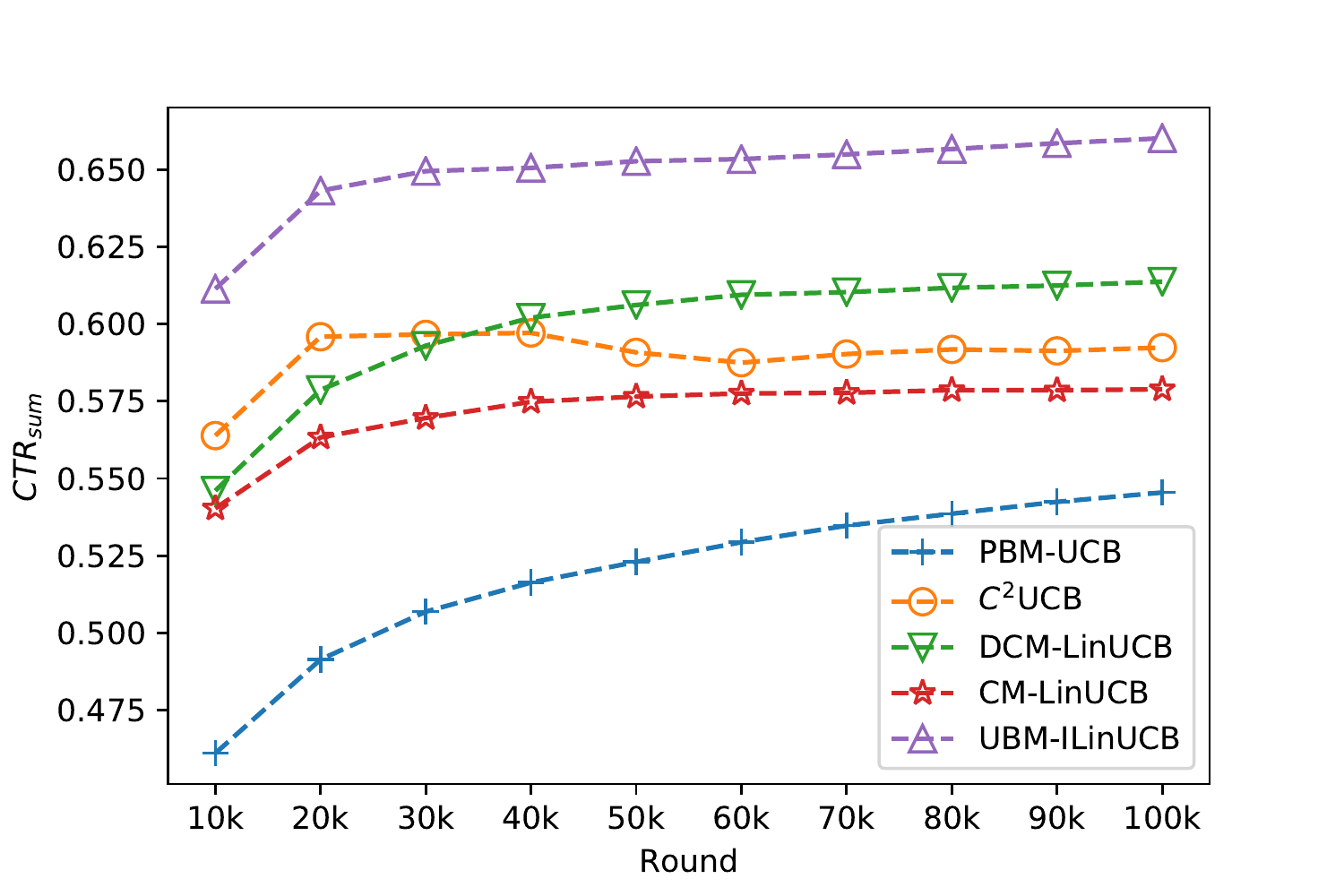}\label{sumk3}}
 \subfigure[$K=3$, $CTR_{set}$]{\includegraphics[width=.24\textwidth]{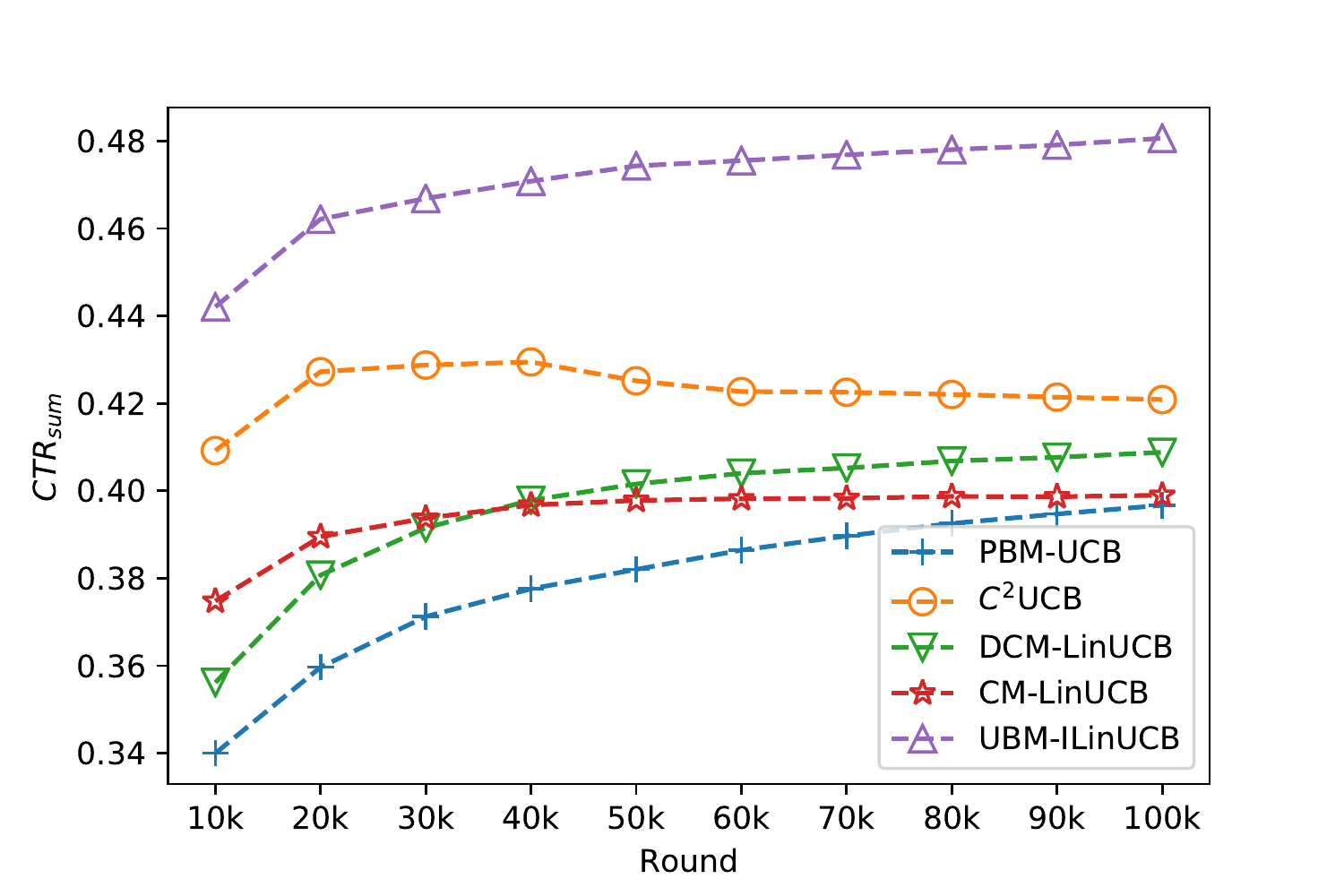}\label{setk3}}
 \subfigure[$K=6$, $CTR_{sum}$]{\includegraphics[width=.24\textwidth]{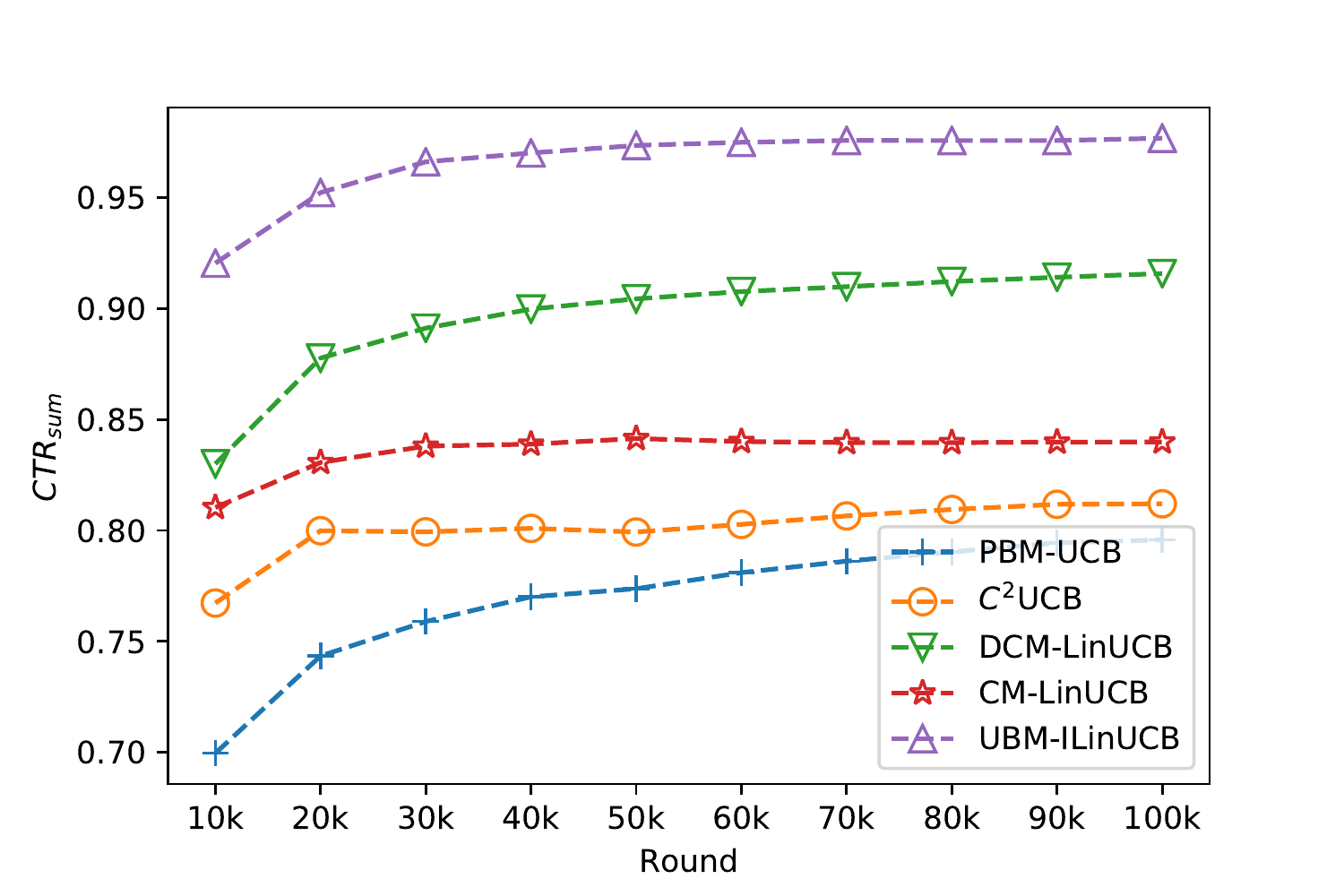}\label{sumk6}}
 \subfigure[$K=6$, $CTR_{set}$]{\includegraphics[width=.24\textwidth]{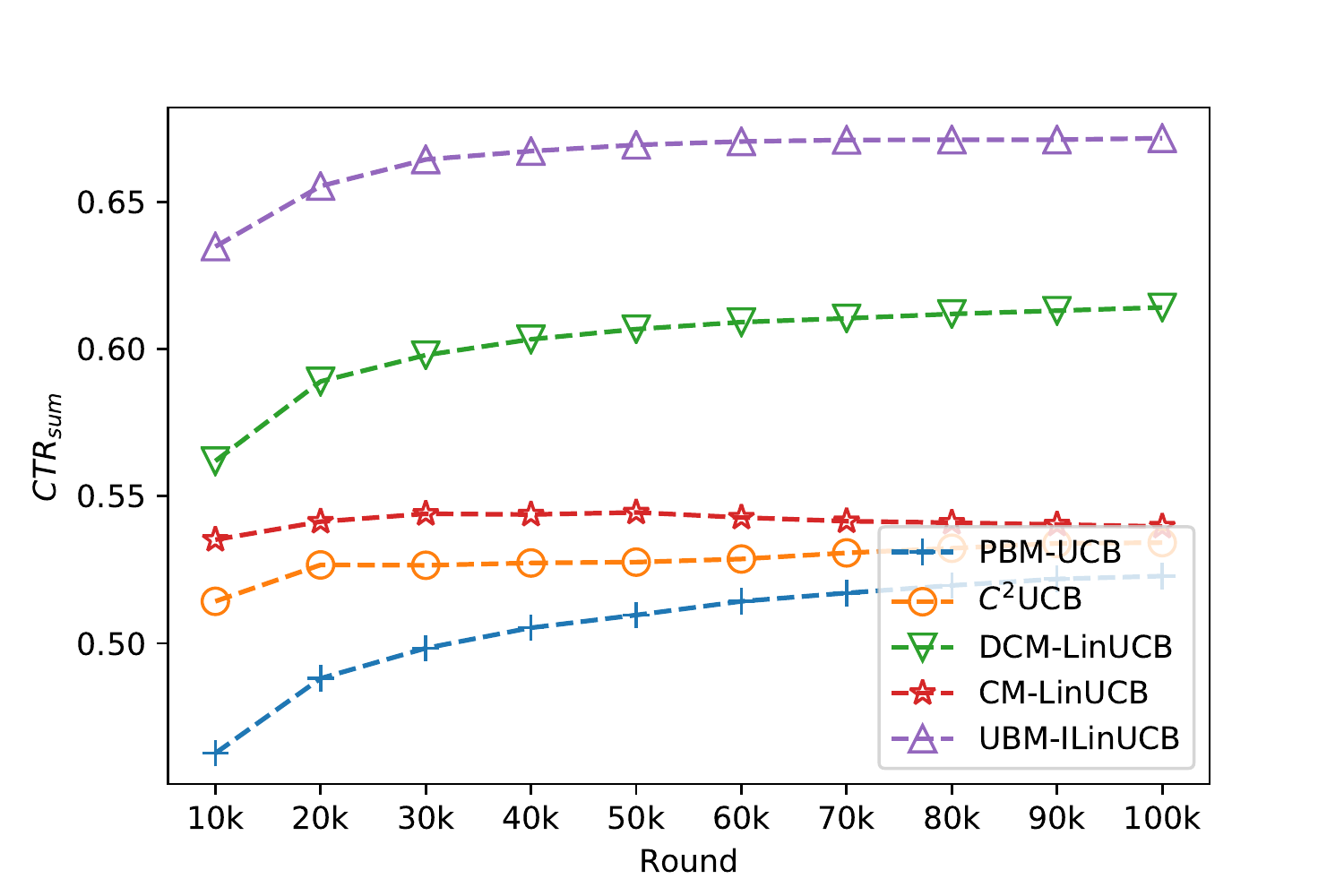}\label{setk6}}
 \caption{Performance of each algorithm under two CTR metrics with the increase of the number of rounds.}\label{convergence}
\end{figure*}

\begin{table*}[!t]
    \centering
    \caption{CTRs with the change of $K$ from $3$ to $6$ when $\alpha$ is optimal. The numbers with a percentage is the CTR lift.}
    \scalebox{1.3}{
    \begin{tabular}{c|c|c|c|c|c|c|c|c}
    \hline
    \hline
    \multirow{2}{*}{Algorithm} & \multicolumn{2}{c|}{$K=3$} & \multicolumn{2}{c|}{$K=4$}&\multicolumn{2}{c|}{$K=5$}&\multicolumn{2}{c}{$K=6$} \\ \cline{2-9}
    &$CTR_{sum}$ &$CTR_{set}$ &$CTR_{sum}$ &$CTR_{set}$ &$CTR_{sum}$ &$CTR_{set}$ &$CTR_{sum}$ &$CTR_{set}$\\
    \hline
    \multirow{2}{*}{C\textsuperscript{2}UCB} &0.592 &0.420 &0.643 &0.440 &0.747 &0.510 &0.812 &0.534\\
    &0\%&0\%&0\%&0\%&0\%&0\%&0\%&0\%\\
    \hline
    \multirow{2}{*}{PBM-UCB} &0.545 &0.396 &0.655 &0.453 &0.744 &0.519 &0.795 &0.522\\
    &-7.9\%&-5.7\%&1.8\%&2.9\%&-0.4\%&1.7\%&-2.1\%&-2.1\%\\
    \hline
    \multirow{2}{*}{CM-LinUCB} &0.578 &0.398 &0.674 &0.464 &0.771 &0.548 &0.839 &0.539\\
    &-2.3\%&-5.2\%&4.8\%&5.4\%&3.2\%&7.4\%&3.3\%&9.3\%\\
    \hline
    \multirow{2}{*}{DCM-LinUCB} &0.613 &0.408 &0.704 &0.474 &0.820 &0.570 &0.915 &0.614\\
    &3.5\%&-2.8\%&9.5\%&7.7\%&9.7\%&11.7\%&12.7\%&14.9\%\\
    \hline
    \multirow{2}{*}{UBM-LinUCB} &\textbf{0.660} &\textbf{0.480} &\textbf{0.781} &\textbf{0.538} &\textbf{0.869} &\textbf{0.604} &\textbf{0.976} &\textbf{0.671}\\
    &\textbf{11.4\%}&\textbf{14.2\%}&\textbf{21.4\%}&\textbf{22.2\%}&\textbf{16.3\%}&\textbf{18.4\%}&\textbf{20.2\%}&\textbf{25.6\%}\\
    \hline
    \hline
    \end{tabular}}
    \label{ctr}
\end{table*}

\subsubsection{Feature Reduction}
Additionally, we reduce the dimension of features inspired by \cite{li2010contextual}, since the time of computation is closely related to the dimension of context. The update equation of $A_k$ takes $O(d^2)$ time and the inverse of $A_k$ need $O(d^3)$ time in each round. Thus, in order to improve the scalability of the algorithm, we apply a pre-trained denoising autoencoder \cite{vincent2008extracting} to reduce the dimension of the feature space from 56 to 10. The encoder contains 4 hidden layers with 25,10,10,25 neural units respectively and an output layer with 56 units. The ReLU activation function is applied to the outputs of the hidden layers. The context vectors $x$ are perturbed by the vectors randomly drawn from a uniform distribution between 0 and 1, formally, $input = 0.95*x + 0.05*noise$. The objective of this autoencoder is to minimize the MSE loss $\|x-output\|^2_2$ between the context and the output. We adopt the output of the second hidden layer as extracted features, whose dimension is 10. The extracted features are utilized as the context in our algorithm and other compared methods.

\subsubsection{Result}
We conduct offline and online experiments to evaluate various algorithms. The first part of experiments illustrates the curves of algorithms with the increase of the number of rounds when $K=3$ and $6$. The second set of experiments shows the performance of different algorithms with the change in the number of recommended commodities. Finally, we implement the online experiment for three days in a real platform and report the performance in different days. The UBM-LinUCB algorithm performs the best in most cases.

\begin{figure}[!t]
 \centering
 \subfigure[$CTR_{sum}$]{\includegraphics[width=.45\textwidth]{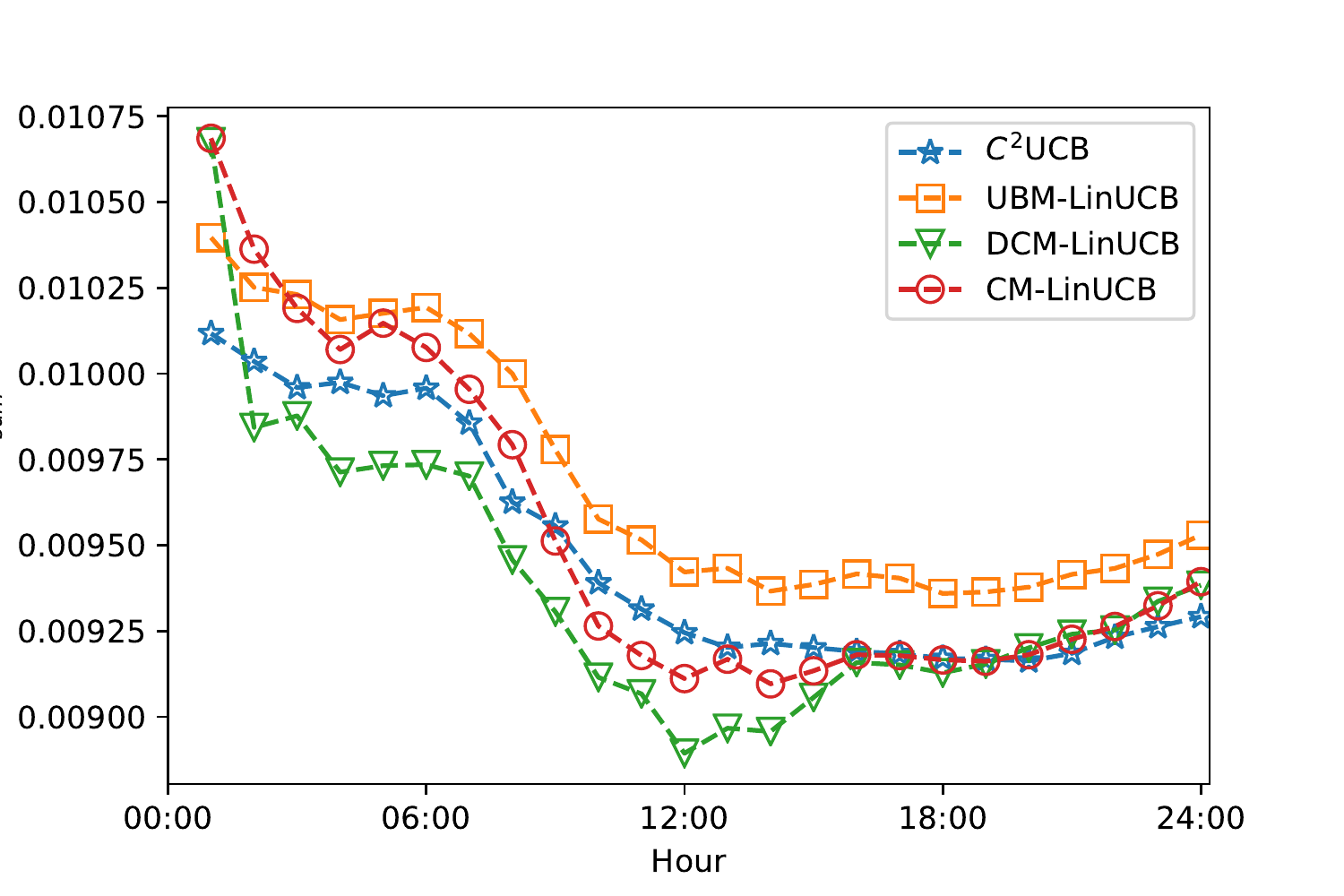}\label{online_sum}}
 \subfigure[$CTR_{set}$]{\includegraphics[width=.45\textwidth]{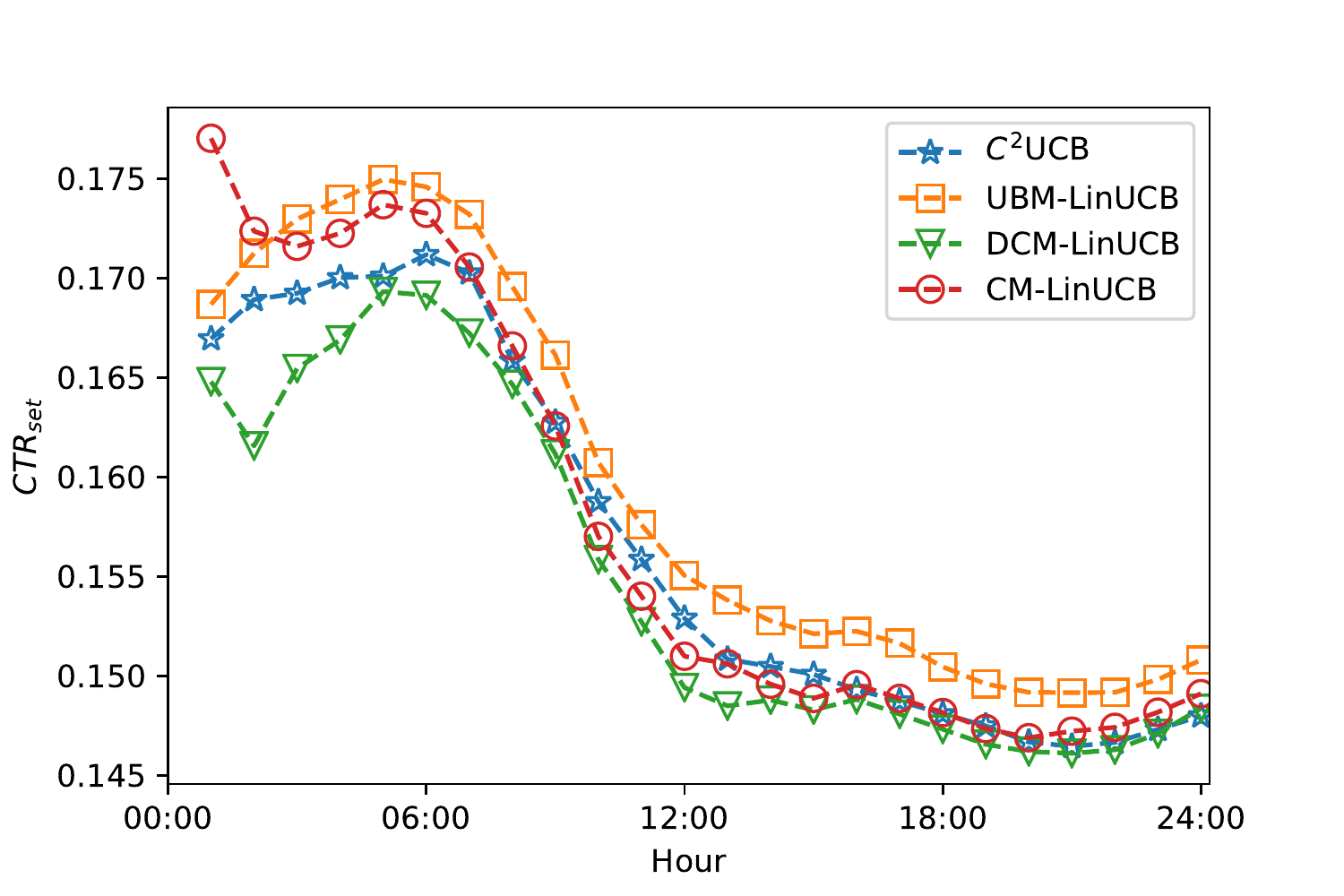}\label{online_set}}
 \caption{Trends in two metrics in the online experiment.} \label{online_fig}
\end{figure}
\begin{figure}[!t]
 \centering
 \subfigure[$CTR_{sum}$]{\includegraphics[width=.45\textwidth]{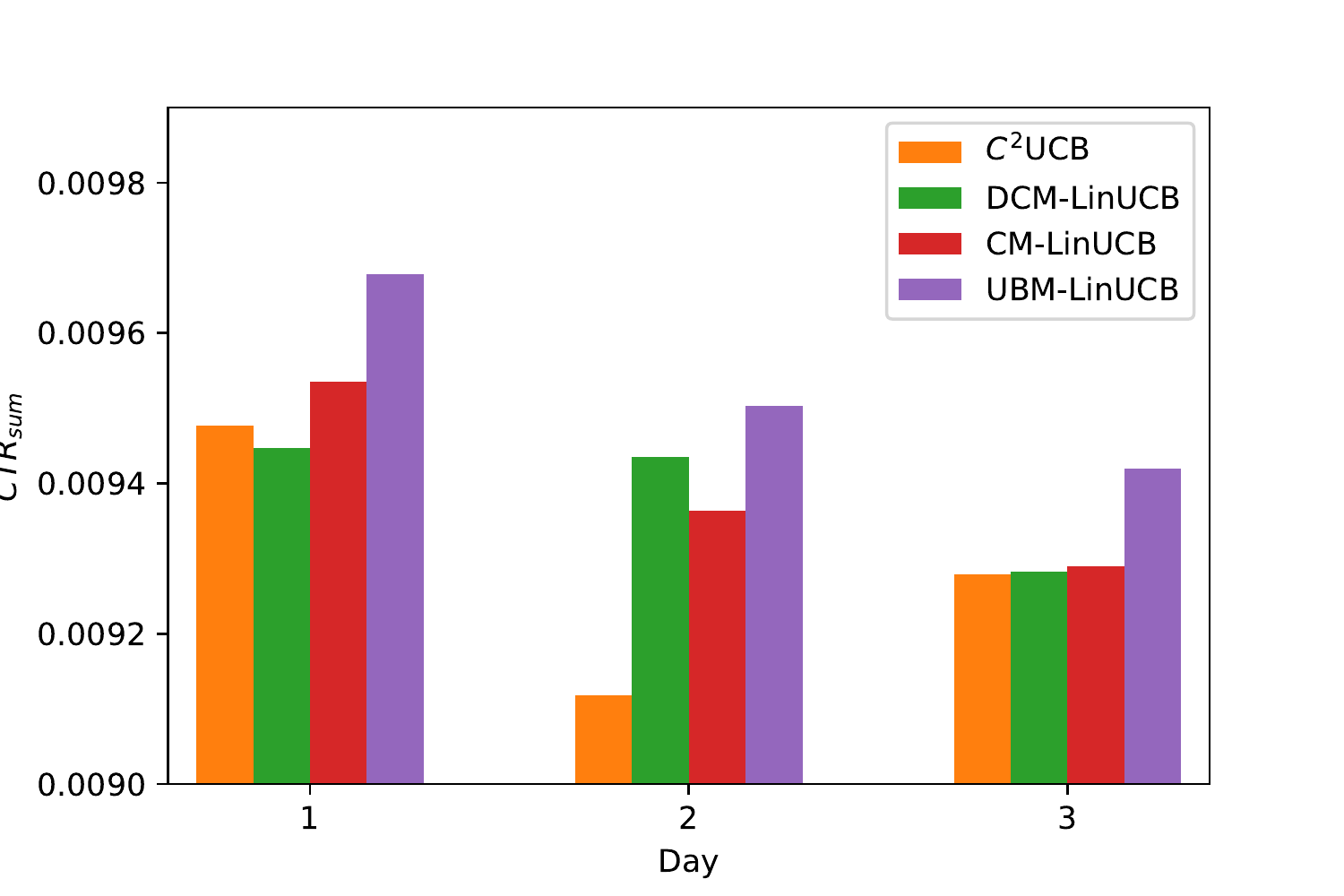}\label{online_sum_bar}}
 \subfigure[$CTR_{set}$]{\includegraphics[width=.45\textwidth]{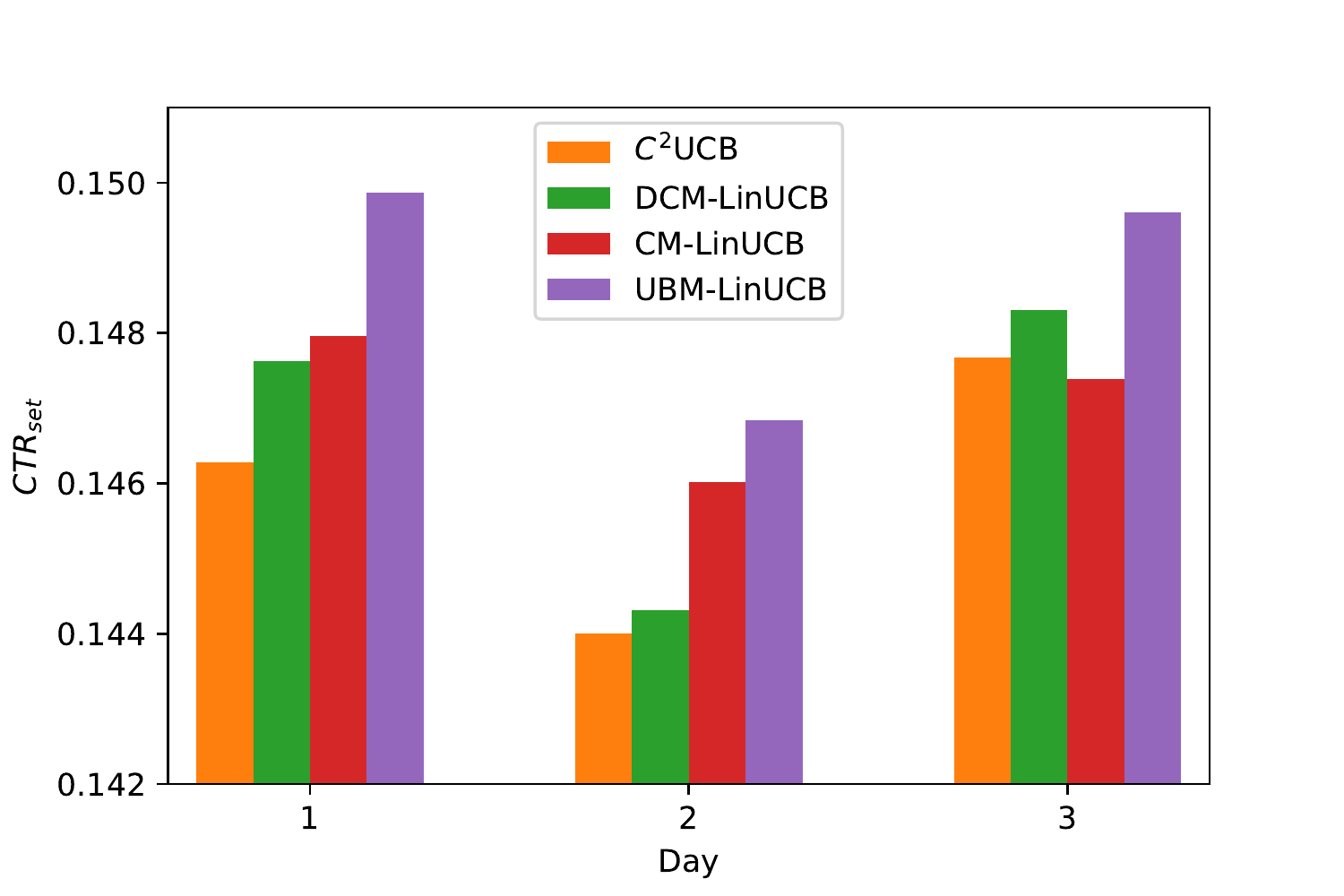}\label{online_set_bar}}
 \caption{CTRs of 3 days in the online experiment.} \label{online_fig_bar}
\end{figure}

\begin{figure}[!t]
 \centering
 \includegraphics[width=.4\textwidth]{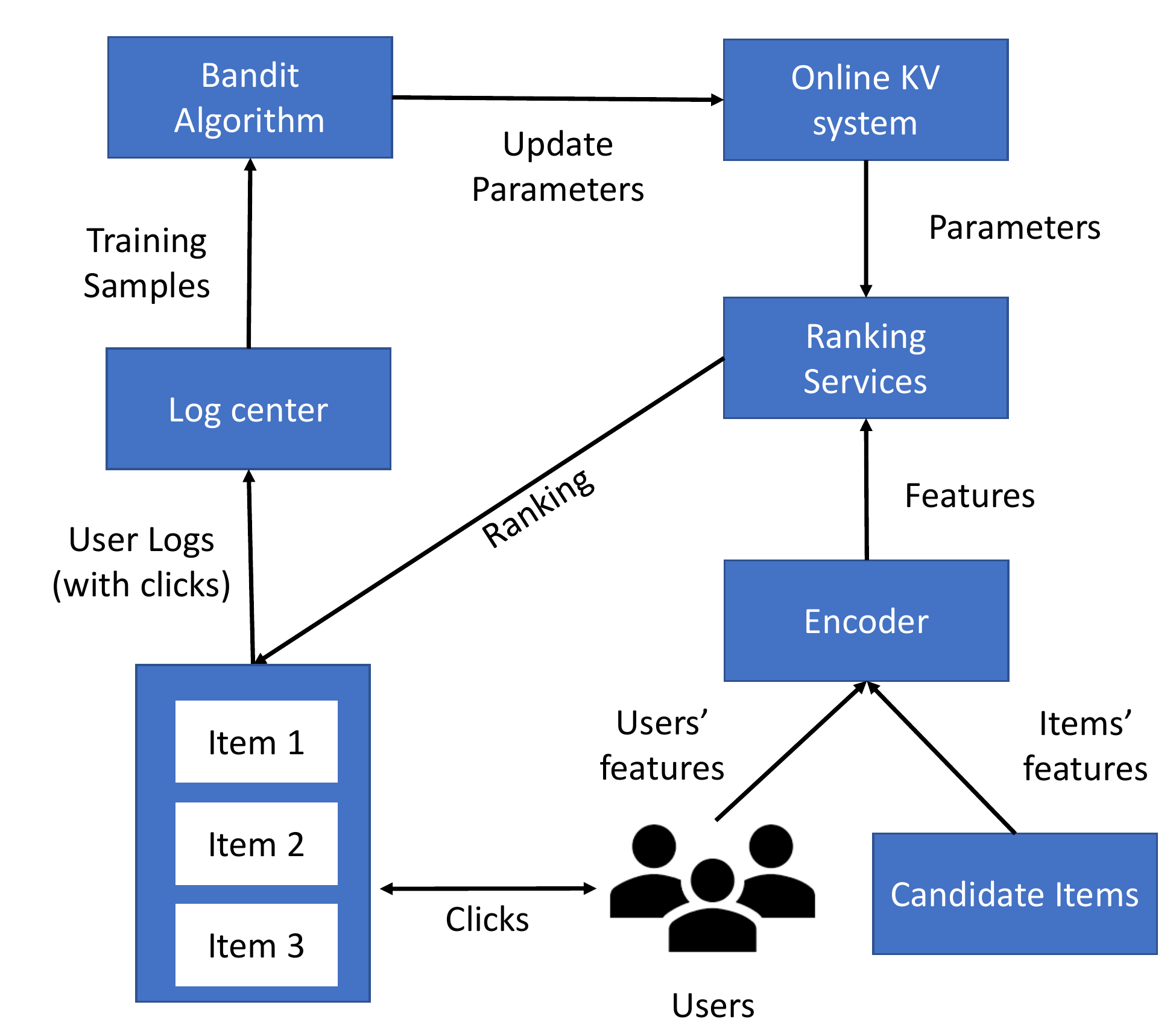}
 \caption{The ranking system of our online experiment. The parameters of our algorithm are stored in the online KV system and used to rank items when users enter our scenario.} \label{flow_structure}
\end{figure}

\textbf{Performance with the increase of the number of rounds. }
Fig. \ref{convergence} demonstrates the performance of our algorithm under two metrics. Sub-figures \ref{sumk3} and \ref{sumk6} are the trends under the $CTR_{sum}$ metric when $K =3$ and 6 respectively. The other two sub-figures show the performance under the $CTR_{set}$ metric. 
The UBM-LinUCB algorithm performs the best in all of these figures. Since the PBM-UCB algorithm does not consider the features of items, it performs worse than other methods. The speed of convergence is similar for the contextual bandit algorithms, while PBM-UCB is relatively slow. This is coincident to theoretical analysis that the regret of contextual bandit algorithms is independent of the number of items rather than non-contextual ones.

When $K=6$, the speed of convergence is faster for almost all the algorithms. For example, for our UBM-LinUCB algorithm, it takes about 30k rounds to converge when $K=6$ while its performance increases slowly until 60k or 70k rounds when $K=3$. The reason would be related to the number of samples obtained in each round. More specifically, $K=6$ means that our algorithm can receive 6 samples per round, which is double compared to the case when $K=3$. The more samples received in each round, the faster the speed of convergence is. One exception would be the CM-LinUCB. Since it only considers the samples before the first click, not all the samples are used for the update. Thus, the learning rate of CM-LinUCB would be insensitive to the increase of $K$.



\textbf{Experiments for the Number of Recommended Items $K$. }
We run algorithms to recommend various number of items and summarize the result in Table \ref{ctr}. The C\textsuperscript{2}UCB method is considered as a baseline in the table and the number with a percentage is the CTR lift compared to C\textsuperscript{2}UCB. A larger CTR of an algorithm indicates its effective use of data and features. We alter the value of $K$ to demonstrate the influence of the recommended set's size. The UBM-LinUCB algorithm improves two CTR metrics by 20.2\% and 25.6\% when $K=6$.

The advantage of our algorithm becomes more significant with the increase of $K$. The reason is that the influence of positions is not significant when $K$ is small. When $K=3$, the difference of the value of position weight is small. Thus, the CTR of an item is similar in the first three positions, which indicates that 1) a sub-optimal policy can also perform well especially when the recommended items of this sub-optimal policy is the same as the optimal one and the only gap is the order; 2) the bias introduced in learning is also relatively small. Then, baseline algorithms without considering the behavior model of users or using other inaccurate models would obtain good performance (for example C\textsuperscript{2}UCB). However, when $K$ becomes larger, the performance of C\textsuperscript{2}UCB is exceeded by other algorithms that involve click models. Moreover, the performances of different click models can be illustrated more straightforwardly. The performance of CM-LinUCB becomes worse compared with DCM-LinUCB and UBM-LinUCB, since it only considers one click.



\subsection{Online experiment}
We implement an online experiment in Taobao. In our scenario, 12 items should be selected from about 40,000 commodities and displayed on users' mobile devices. Bucket test or called A/B test is adopted. More specifically, users that enter our scenario will be randomly divided into different buckets. Each algorithm is deployed to one bucket and recommends about 5 millions of communities per day in the large-scale online evaluation. We use the same metrics as the offline experiment and report the performance of four algorithms. We do not evaluate PBM-UCB since it is infeasible to store the means and variances of all the 40,000 commodities in our online platform. 

Fig. \ref{flow_structure} illustrates the structure of our online platform. When users enter our scenario, their information will be concatenated with each item's feature. Then, an encoder mentioned in the previous section is used to reduce the dimension of features. In the ranking services, items are ranked according to the features and parameters. More specifically, users are divided randomly to different buckets which correspond to different algorithms. According to the ids of buckets, parameters restored in the online KV system are read and used to compute the rank given features. The results are displayed in a mobile application and users give their feedback by interaction. The log center receives the raw logs and transfers them to samples to train different bandit algorithms. Each algorithm only obtains samples from its own bucket to update parameters.

Fig. \ref{online_fig} shows the trends of cumulative mean of CTRs in three days with the increase of time. Since the number of samples is relatively small in the first few hours, the performance of different algorithms is unstable. Except the perturbation on the first two hours, our algorithm constantly performs the best and the improvement is stable. Notice that all the algorithms performs well before 6 a.m, which is normal in our platform. The reason would be that active users shift with time. More specifically, people who use our app from 0 a.m to 6 a.m usually are keen to shop and thus more likely to click recommended items. DCM-LinUCB and CM-LinUCB perform similarly, possibly because a portion of users only click one item. Since the portion is not fixed, their performances are fluctuated each day. Fig. \ref{online_fig_bar} summarizes the total CTRs in each day. Our algorithm can improve $CTR_{sum}$ and $CTR_{set}$ steadily compared to benchmark algorithms, although each algorithm's CTRs change in different dates. The result is reliable and meaningful considering the complex online environment and the huge volume of the recommended items (about 60 million). Since CTR is usually positively correlated with the turnover and profit of an e-commerce platform, our algorithm can improve the profit directly.


\section{Conclusion}
In this paper, we focus on addressing the position bias and pseudo-exposure problem in a large-scale mobile e-commerce platform. First, we model the online recommendation as a contextual combinatorial bandit problem. Second, to address the pseudo-exposure problem, we utilize the UBM model to estimate probabilities that users view items at different ranks and propose a novel method UBM-LinUCB to take advantage of the estimated probabilities in our linear reward model. Third, we prove a sublinear expected cumulative regret bound $\tilde{O}(d\sqrt{TK})$ for our regret function. Finally, we conduct experiments to evaluate our algorithm in two real-world datasets by an unbiased estimator and an online scenario provided by Taobao, one of the most popular e-commerce platforms in the world.
Results indicate that our algorithm improves CTRs under two metrics in both offline and online experiments and outperforms other algorithms.

\begin{acks}
This work was supported by Alibaba Group through Alibaba Innovative Research (AIR) Program and Alibaba-NTU Joint Research Institute (JRI), Nanyang Technological University, Singapore.
\end{acks}

\bibliographystyle{ACM-Reference-Format}
\bibliography{acmart}

\newpage
\appendix
{\onecolumn 
\section{SUPPLEMENT}
\subsection{Proof of Lemma \ref{lemma_order}}
\begin{proof}
Notice that $F(S_t) = 1$ except none of item is clicked, formally, 
\begin{equation}
\mathbb{E}[F(S_t)]=1-\prod\nolimits_{k=1}^K\left[1-\gamma(a_{k,t})w_{k,0}\right]
\end{equation}
It suffices to show that $\prod_{k=1}^K\left[1-\gamma(a_{k,t})w_{k,0}\right]$ is minimized with $a^*_{1,t},...,a^*_{K,t}$ when the arms with the $k$-th highest value $\gamma(a_{k,t}^*)$ aligns with the order of $w_{k,0}$, i.e., $\gamma(a_{1,t}^*)\geq\gamma(a_{2,t}^*)\geq...\geq\gamma(a_{K,t}^*)$.
\begin{equation*}
\begin{aligned}
\log{\prod\nolimits_{k=1}^K\left(1-\gamma(a_{k,t})w_{k,0}\right)}=&\sum\nolimits_{k=1}^K\log\left(1-\gamma(a_{k,t})w_{k,0}\right)\\
=&-\sum\nolimits_{k=1}^K\sum\nolimits_{n=1}^\infty\frac{\left(\gamma(a_{k,t})w_{k,0}\right)^n}{n}\\
=&-\sum\nolimits_{n=1}^\infty\frac{1}{n}\sum\nolimits_{k=1}^K\gamma^n(a_{k,t})w^n_{k,0}\\
\geq&-\sum\nolimits_{n=1}^\infty\frac{1}{n}\sum\nolimits_{k=1}^K\gamma^n(a_{k,t}^*) w^n_{k,0}\\
=&\log{\prod\nolimits_{k=1}^K(1-\gamma(a_{k,t}^*)w_{k,0})}
\end{aligned}
\end{equation*}
The second equality follows $\log(1-x)=-\sum_{n=1}^\infty\frac{x^n}{n},\forall x\in(0,1)$ and the inequality follows the Rearrangement inequality, which states that $\sum_{i}x_iy_i\geq\sum_{i}x_{\sigma(i)}y_{\sigma(i)}$ for any sequences of real numbers $x_1\geq x_2\geq...\geq x_n$ and $y_1\geq y_2\geq...\geq y_n$, and any permutations $x_{\sigma(1)},...,x_{\sigma(n)}$.
\end{proof}

\subsection{Proof of Lemma \ref{lemma:regret_set}}
\begin{proof}
Recall that $a_{k,t}$ is the selected arm by UBM-LinUCB at round $t$ and position $k$ based on Lemma \ref{lemma_order}. Let $\eta_{k,t}=\gamma(a^*_{k,t})w_{k,0}-\gamma(a_{k,t})w_{k,0} \geq 0$ and divide $\mathbb{E}[F(S_t)]$ into the sum of the probabilities that the first click appears at each position $k \in [1,K]$:
\begin{equation}
\begin{aligned}
&\sum\nolimits_{t=1}^T\mathbb{E}[F(S_t)]\\
&=\sum\nolimits_{t=1}^T\sum\nolimits_{k=1}^K\prod\nolimits_{j=1}^{k-1}\left[1-\left(\gamma(a^*_{j,t})w_{j,0}-\eta_{j,t}\right)\right]\left(\gamma(a^*_{k,t})w_{k,0}-\eta_{k,t}\right)\\
&\geq \sum\nolimits_{t=1}^T\sum\nolimits_{k=1}^K\prod\nolimits_{j=1}^{k-1}\left(1-\gamma(a^*_{j,t})w_{j,0}\right)\left(\gamma(a^*_{k,t})w_{k,0}-\eta_{k,t}\right)\\
&\geq\sum\nolimits_{t=1}^T\sum\nolimits_{k=1}^K \left[ \prod\nolimits_{j=1}^{k-1}\left(1-\gamma(a^*_{j,t})w_{j,0}\right)\gamma(a^*_{k,t})w_{k,0}-\eta_{k,t} \right]\\
&=\sum\nolimits_{t=1}^T\mathbb{E}[F(S^*_t)]-\sum\nolimits_{t=1}^T\sum\nolimits_{k=1}^K\eta_{k,t}
\end{aligned}
\end{equation}
Thus, we have:
\begin{equation}
\begin{aligned}
R(T)\leq&\sum\nolimits_{t=1}^T\sum\nolimits_{k=1}^K\eta_{k,t}\\
=& \sum\nolimits_{t=1}^T\sum\nolimits_{k=1}^Kw_{k,0}\left[\gamma(a^*_{k,t})-\gamma(a_{k,t})\right] \\
\end{aligned}
\end{equation}
Then, we finish the proof.
\end{proof}

\subsection{Proof of Theorem \ref{theorem:regret_sum}} \label{sec_proof}
Assume that (1) $r_{k,a} = w_{k,0} \theta^* x_a$, (2) $1 \leq \sum_{k=1}^K w_{k,k-1}^2 = \phi_w' < K$, (3) $\|x\|_2 \leq 1$ and (4) $\|\theta^*\|_2^2 \leq \beta$. The upper bound of regret is defined by Lemma \ref{lemma:regret_set}:
$$R(T) \leq \mathbb{E}\left [ \sum_{t=1}^T \sum_{k=1}^K w_{k,0}\left [\gamma(a_{k,t}^*)-\gamma(a_{k,t})\right ]\right ]$$
where $\gamma(a) = \theta^* x_a$ is the attractiveness of an item $a$ and $\theta^*$ is the optimal parameter. 
\begin{proof}
We define the event 
\begin{align*}
E = &\{|\langle x_{a_{k,t-1}}, \theta^*-\theta_{t-1} \rangle | \leq \alpha \sqrt{x_{a_{k,t-1}}^T A_{t-1}^{-1} x_{a_{k,t-1}}}, \\
& \forall {a_{k,t}}\in A_t, \forall t\leq T, \forall k \leq K \}
\end{align*} 
where $\langle \cdot, \cdot \rangle$ means the dot product and $\bar{E}$ is the complement of $E$. Then notice that $P(E) \leq 1$ and $w_{k,0}\left [\gamma(a_{k,t}^*)-\gamma(a_{k,t})\right ] \leq 1$, we have
\begin{align*}
    R(T) & \leq P(E) \mathbb{E} \left  [\sum_{t=1}^T \sum_{k=1}^K w_{k,0}\left [\gamma(a_{k,t}^*)-\gamma(a_{k,t})\right ] \middle| E\right ] \\
    & + P(\bar{E})\mathbb{E} \left  [\sum_{t=1}^T\sum_{k=1}^K w_{k,0}\left [\gamma(a_{k,t}^*)-\gamma(a_{k,t})\right ] \middle| \bar{E} \right ] \\
    & \leq \mathbb{E} \left  [\sum_{t=1}^T \sum_{k=1}^K w_{k,0}\left [\gamma(a_{k,t}^*)-\gamma(a_{k,t})\right ] \middle| E\right ] + TKP(\bar{E})
\end{align*}
Using the definition of event $E$, we have 
$$ \gamma(a_{k,t}^*) = \langle x_{a_{k,t}^*}, \theta^*\rangle \leq \langle x_{a_{k,t}^*}, \theta_{t-1}\rangle + \alpha \sqrt{x_{a_{k,t}^*}^T A_{t-1}^{-1} x_{a_{k,t}^*}} $$
under event $E$. Since $a_{k,t}$ is the arm that our algorithm chooses in round t, we have
\begin{align*}
    &\langle x_{a_{k,t}^*}, \theta_{t-1}\rangle + \alpha \sqrt{x_{a_{k,t}^*}^T A_{t-1}^{-1} x_{a_{k,t}^*}}\\
    \leq& \langle x_{a_{k,t}}, \theta_{t-1}\rangle + \alpha \sqrt{x_{a_{k,t}}^T A_{t-1}^{-1} x_{a_{k,t}}}
\end{align*}
Thus, 
$$ \gamma(a_{k,t}) \leq \gamma(a_{k,t}^*) \leq \langle x_{a_{k,t}}, \theta_{t-1}\rangle + \alpha \sqrt{x_{a_{k,t}}^T A_{t-1}^{-1} x_{a_{k,t}}}$$
Using this property and $\gamma(a) = \theta^* x_a$, we have 
\begin{align*}
    \gamma(a_{k,t}^*)-\gamma(a_{k,t}) &\leq \langle x_{a_{k,t}}, \theta_{t-1}-\theta^*\rangle + \alpha \sqrt{x_{a_{k,t}}^T A_{t-1}^{-1} x_{a_{k,t}}} \\
    &\leq 2\alpha \sqrt{x_{a_{k,t}}^T A_{t-1}^{-1} x_{a_{k,t}}}
\end{align*}
The second inequality is based on the definition of event $E$. Thus, we have
$$ R(T) \leq 2 \alpha \mathbb{E} \left[ \sum_{t=1}^T \sum_{k=1}^K w_{k,0} \sqrt{x_{a_{k,t}}^T A_{t-1}^{-1} x_{a_{k,t}}} \right] + TKP(\bar{E})$$
Using the bounds proved in the next two sections, we have 
$$ R(T) \leq 2 \alpha \sqrt{2TKd\ln\left(1+\frac{\phi_w' T}{\lambda d}\right)} + TK\delta$$
when $\lambda \geq \phi_w' \geq 1$ and  $$\alpha \geq \sqrt{d\ln\left(1+ \frac{\phi_w' T}{d\lambda}\right) + 2\ln\left(\frac{1}{\delta} \right)}+\sqrt{\lambda \beta}.$$
Let $\delta = \frac{1}{TK}$, we finish the proof.
\end{proof}

\subsection{Bound of $\sum_{t=1}^T \sum_{k=1}^K w_{k,0} \sqrt{x_{a_{k,t}}^T A_{t-1}^{-1} x_{a_{k,t}}}$}
Let $z_{k,t} = \sqrt{x_{a_{k,t}}^T A_{t-1}^{-1} x_{a_{k,t}}}$.
\begin{lemma} \label{lemma_sqrt}
For any $\lambda \geq \phi_w', \ \forall k' \ s.t. \ 0 \leq k' \leq k-1$, then for any time $T$, $$\sum_{t=1}^T \sum_{k=1}^K w_{k,k'} z_{k,t} \leq \sqrt{2TKd\ln\left(1+\frac{\phi_w' T}{\lambda d}\right)}$$
\end{lemma}
\begin{proof}
From Cauchy-Schwarz inequality, we give an upper bound 
\begin{align*}
    \sum_{t=1}^T \sum_{k=1}^K w_{k,0} z_{k,t} &\leq \sum_{t=1}^T \sum_{k=1}^K w_{k,k'} z_{k,t} \\
    &\leq \sqrt{TK}\sqrt{\sum_{t=1}^T \sum_{k=1}^K w_{k,k'}^2 z_{k,t}^2}
\end{align*}
$,\forall k' \ s.t. \ 0 \leq k' \leq k-1$. The first inequation uses the property of $w_{k,k'}$ in the main body of the paper to leverage the following lemma (cannot be used for $w_{k,0}$ due to the definition of $A_t$):
\begin{lemma}\label{lemma_sum}
(Lemma 4.4 \cite{li2016contextual}) If $\lambda \geq \sum_{k=1}^K w_{k,k'}^2, \ \forall k' \ s.t. \ 0 \leq k' \leq k-1$, then for any time $T$, 
$$\sum_{t=1}^T \sum_{k=1}^K w_{k,k'}^2 z_{k,t}^2 \leq 2d\ln\left(1+\frac{1}{\lambda d} \sum_{t=1}^T \sum_{k=1}^K w_{k,k'} \right) $$
\end{lemma}

Using Lemma \ref{lemma_sum} and the property that $$\sum_{t=1}^T \sum_{k=1}^K w_{k,k'} \leq \sum_{t=1}^T \sum_{k=1}^K w_{k,k-1}$$, the upper bound is 
\begin{equation}
    \sum_{t=1}^T \sum_{k=1}^K w_{k,k'} z_{k,t} \leq \sqrt{2TKd\ln\left(1+\frac{\phi_w' T}{\lambda d}\right)}
\end{equation}
\end{proof}

\subsection{Bound of $P(\bar{E})$}
\begin{lemma} \label{lemma_e}
For any $\delta \in (0,1), \lambda \geq \phi_w' \geq 1$, $ \beta \geq \|\theta^*\|_2^2$, if $$\alpha \geq \left[ \sqrt{2\ln\left(1+ \frac{\phi_w' T}{2}\right) + 2\ln\left(\frac{1}{\delta} \right)}+\sqrt{\lambda \beta}\right],$$ we have $P(\bar{E})\geq \delta$.
\end{lemma}
\begin{proof}
We first define $$\bar{\eta}_{k,t} = w_{k,k'} \left[\gamma(a_{k,t})-\gamma_t(a_{k,t})\right]$$ to use the Theorem 1 of \cite{abbasi2011improved}, where $\gamma_t(a_{k,t})=\theta_t x_{k,t}$ and $\gamma(a_{k,t})=\theta^* x_{k,t}$. $\bar{\eta}_{k,t}$ is a Martingale Difference Sequence and bounded by $[-1,1]$. Thus, it is a conditionally sub-Gaussian with constant $R=1$. We also define that 
\begin{align*}
    S_t &= \sum_{t=1}^T \sum_{k=1}^K w_{k,k'} x_{a_{k,t}}\bar{\eta}_{k,t} \\
    &= b_t - \sum_{t=1}^T \sum_{k=1}^K w_{k,k'}^2 x_{a_{k,t}} \gamma_t(a_{k,t}) \\
    &=  b_t - \left[\sum_{t=1}^T \sum_{k=1}^K w_{k,k'}^2 x_{a_{k,t}}x_{a_{k,t}}^T\right]\theta^*
\end{align*}
where the second equation uses the definition of $b_t$ and $\mathbb{E}[r_{a_{k,t}}] = w_{k,k'} \gamma(a_{k,t})$.
Then, according to the Theorem 1 of \cite{abbasi2011improved}, for any $\delta\in(0,1)$, with probability at least $1-\delta$,
\begin{equation}
    \|S_t\|_{A_t^{-1}}\leq\sqrt{2\ln\left(\frac{\det(A_t)^{\frac{1}{2}}\det(A_0)^{-\frac{1}{2}}}{\delta}\right)}
\end{equation}
where $\|S_t\|_{A_t^{-1}} = \sqrt{S_t^T A_t^{-1} S_t}$.
Moreover, using the trace-determinant inequality and the assumption 1) $\|x_a\| \leq 1$, 2) $w_{k,k'} \leq w_{k,k-1}$, we have 
$$\det(A_t)^{\frac{1}{d}} \leq \frac{trace(A_t)}{d} = \lambda + \frac{1}{d} \sum_{t=1}^T \sum_{k=1}^K w_{k,k'}^2 \|x_{a_{k,t}}\|_2^2 \leq \lambda + \frac{\phi_w' T}{d}$$
Since $A_0 = \lambda I$, we have 
\begin{equation} \label{upper_S}
    \|S_t\|_{A_t^{-1}}\leq \sqrt{d\ln\left(1 + \frac{\phi_w' T}{d\lambda}\right) + 2\ln\left(\frac{1}{\delta} \right)}
\end{equation}

Since $\lambda \geq 1$, notice that 
\begin{align*}
    A_t\theta_t &= b_t = S_t + \left[\sum_{t=1}^T \sum_{k=1}^K w_{k,k'}^2 x_{a_{k,t}}x_{a_{k,t}}^T\right]\theta^* \\
    &= S_t + (A_t-\lambda I)\theta^* \\
\end{align*}
Thus, $$\theta_t-\theta^* \leq A_t^{-1}(S_t-\lambda \theta^*)$$
Then, for any $a\in A, k \leq K$, based on Cauchy-Schwarz inequality and triangle inequality, we have
\begin{align*}
    |\langle x_{a_{k,t}}, \theta_{t}-\theta^*\rangle| &\leq | x_a^T A_t^{-1}(S_t-\lambda\theta^*)| \\
    &\leq \| x_a^T\|_{A_t^{-1}} \|(S_t-\lambda\theta^*)\|_{A_t^{-1}} \\
    &\leq \| x_a^T\|_{A_t^{-1}} \left[ \|S_t\|_{A_t^{-1}}+\lambda\|\theta^*\|_{A_t^{-1}}\right]\\
    &\leq \| x_a^T\|_{A_t^{-1}} \left[ \|S_t\|_{A_t^{-1}}+\sqrt{\lambda K}\right]
\end{align*}

where we use the property $$\|\theta^*\|_{A_t^{-1}}^2 \leq \frac{1}{Eig_{\min}(A_t)} \|\theta^*\|_2^2 \leq \frac{1}{\lambda} \|\theta^*\|_2^2 \leq \frac{\beta}{\lambda} $$ in the last inequality and $Eig(A_t)$ is the eigenvalue of $A_t$. Then using the Eq. (\ref{upper_S}), with probability $1-\delta$, 
\begin{align*}
    &|\langle x_{a_{k,t}}, \theta_{t-1}-\theta^*\rangle| \\
    \leq& \|x_a^T\|_{A_t^{-1}} \left[ \sqrt{d\ln\left(1+ \frac{\phi_w' T}{d\lambda}\right) + 2\ln\left(\frac{1}{\delta} \right)}+\sqrt{\lambda \beta}\right]
\end{align*}
Recall the definition of $E$, when $$\alpha \geq \sqrt{d\ln\left(1+ \frac{\phi_w' T}{d\lambda}\right) + 2\ln\left(\frac{1}{\delta} \right)}+\sqrt{\lambda \beta},$$
then $P(E)\leq 1-\delta$ and thus $P(\bar{E})\geq \delta$.
\end{proof}



\section{UBM estimator} \label{UBM_est}
Let $\langle W, \Phi(a,\cdot,\cdot|X) \rangle = \sum_{k=1}^K \sum_{k'=0}^k w_{k,k'} \Phi(a,k,k'|X)$, where $w_{k,k'}$ is the position bias for the position $k$ when the item at rank $k'$ is clicked and $k'=0$ if there is no click before position $k$. Then, we have: 
\begin{align*}
    V_{UBM}(\Phi) &= \mathbb{E}_X\left[ \mathbb{E}_{S\sim \Phi(\cdot|X)\atop r\sim D(\cdot|X)}\left[\sum_{k=1}^K \sum_{k'=0}^k r(a_k,k,k'|X) \right]\right] \\
    &= \mathbb{E}_X\left[ \sum_{S} \Phi(S|X)\left[\sum_{k=1}^K \sum_{k'=0}^k  \sum_{a \in \mathcal{A}} r(a,k,k'|X) \mathbbm{1}\{a_k=a, c_{last}=k'\}  \right] \right]\\
    &= \mathbb{E}_X\left[ \sum_{k=1}^K \sum_{k'=0}^k \sum_{a \in \mathcal{A}} r(a,k,k'|X) \sum_{S} \Phi(S|X) \mathbbm{1}\{a_k=a, c_{last}=k'\}\right]\\
    &= \mathbb{E}_X\left[\sum_{k=1}^K \sum_{k'=0}^k \sum_{a \in \mathcal{A}} r(a,k,k'|X) \Phi(a,k,k'|X)\right] \\
    &= \mathbb{E}_X\left[\sum_{a \in \mathcal{A}} \gamma(a|X) \sum_{k=1}^K \sum_{k'=0}^k w_{k,k'} \Phi(a,k,k'|X) \right]\\ 
    &= \mathbb{E}_X\left[\sum_{a \in \mathcal{A}} \gamma(a|X) \sum_{k=1}^K \sum_{k'=0}^k w_{k,k'} \Phi(a,k,k'|X) \right]\\
    &= \mathbb{E}_X\left[\sum_{a \in \mathcal{A}} \gamma(a|X) \langle \tilde{W}, \pi(a,\cdot,\cdot|X) \rangle \frac{\langle \tilde{W}, \Phi(a,\cdot,\cdot|X)\rangle}{\langle W, \pi(a,\cdot,\cdot|X)\rangle}\right]\\
    &= \mathbb{E}_X\left[ \mathbb{E}_{S\sim \pi(\cdot|X)\atop r\sim D(\cdot|X)}\left[\sum_{k=1}^K \sum_{k'=0}^k r(a_k,k,k'|X)\frac{\langle \tilde{W}, \Phi(a_k,\cdot,\cdot|X)\rangle}{\langle \tilde{W}, \pi(a_k,\cdot,\cdot|X)\rangle}\right]\right]
\end{align*}
where $c_{last}$ is the position of the last click before $k$.
Thus, with the assumption that users' behaviors follow the UBM model, the estimator is unbiased.
}
\end{document}